\documentclass[11pt,preprint]{imsart}

\usepackage{algorithm}
\usepackage{algpseudocode}
\usepackage{graphicx} 
\usepackage{amsmath, bm, amssymb, amsthm, xcolor}
\RequirePackage[round,colon,authoryear]{natbib}
\RequirePackage[colorlinks,citecolor=blue,urlcolor=blue]{hyperref}

\newcommand{\hb}{\bm{\hat \beta}}

\newcommand{\ld}{\dot{\ell}}
\newcommand{\ldd}{\ddot{\ell}}
\newcommand{\bld}{\dot{\bm{\ell}}}
\newcommand{\bldd}{\ddot{\bm{\ell}}}
\newcommand{\brdd}{\ddot{\bm{r}}}
\newcommand{\Ldd}{\ddot{\bm{L}}}
\newcommand{\Rdd}{\ddot{\bm{R}}}
\newcommand{\bLdd}{\bar{\ddot{\bm{L}}}}
\newcommand{\bRdd}{\bar{\ddot{\bm{R}}}}

\newcommand{\rdd}{\ddot{r}}

\newcommand{\lddd}{\dddot{\ell}}

\newcommand{\diag}{\bm{{\rm diag}}}
\newcommand{\hbx}[1]{\bm{ \hat \beta}_{\backslash #1}}
\newcommand{\hbm}{\bm{ \hat \beta}_{\backslash \mathcal{M} }}
\newcommand{\hbi}{\bm{ \hat \beta}_{\backslash i }}

\newcommand{\bbf}{\bm{f}}
\newcommand{\ba}{\bm{a}}
\newcommand{\bx}{\bm{x}}

\newcommand{\bw}{\bm{w}}
\newcommand{\bu}{\bm{u}}
\newcommand{\bv}{\bm{v}}
\newcommand{\bvm}{\bm{v}_{\cM}}
\newcommand{\bD}{\bm{D}}
\newcommand{\bA}{\bm{A}}
\newcommand{\bX}{\bm{X}}
\newcommand{\bXm}{\bm{X}_{\backslash \cM}}

\newcommand{\bbXm}{\bar{\bm{X}}_{\backslash \cM}}
\newcommand{\bGamma}{\bm{\Gamma}}

\newcommand{\polylog}{{\rm polylog}}

\newcommand{\bbeta}{\bm{\beta}}

\newcommand{\ev}{\bm e}

\newcommand{\tbm}{\tilde{\bm \beta}_{ \backslash \mathcal{M}}}

\newcommand{\argmin}{\arg\min}
\newcommand{\RR}{\mathbb{R}}

\newcommand{\PP}{\mathbb{P}}
\newcommand{\EE}{\mathbb{E}}
\newcommand{\II}{\mathbb{I}}
\newcommand{\cA}{\mathcal{A}}
\newcommand{\tA}{\tilde{A}}
\newcommand{\cD}{\mathcal{D}}
\newcommand{\cM}{\mathcal{M}}
\newcommand{\cDm}{\mathcal{D}_{\backslash \cM}}
\newcommand{\cT}{\mathcal{T}}
\newcommand{\cX}{\mathcal{X}}
\newcommand{\calM}{\cM}
\newcommand{\bb}{\bm{b}}
\newcommand{\bxi}{\bm{\xi}}

\newcommand{\bG}{\bm{G}}
\newcommand{\bGm}{\bm{G}_{\backslash\cM}}
\newcommand{\bbG}{\bar{\bm{G}}}
\newcommand{\bbGm}{\bar{\bm{G}}_{\backslash\cM}}
\newcommand{\bM}{\bm{M}}
\newcommand{\bI}{\bm{I}}
\newcommand{\bZ}{\bm{Z}}

\newcommand{\bSigma}{\bm{\Sigma}}
\newcommand{\cN}{\mathcal{N}}
\newcommand{\cB}{\mathcal{B}}
\newcommand{\bzero}{\bm{0}}
\newcommand{\eps}{\epsilon}
\newcommand\independent{\protect\mathpalette{\protect\independenT}{\perp}}
\def\independenT#1#2{\mathrel{\rlap{$#1#2$}\mkern2mu{#1#2}}}

\newtheorem{theorem}{Theorem}[section]
\newtheorem{lemma}[theorem]{Lemma}

\newtheorem{assumptionA}{Assumption}

\newtheorem{assumptionB}{Assumption}

\newtheorem{definition}[theorem]{Definition}
\newtheorem{remark}{Remark}[theorem]
\newtheorem{example}{Example}[section]

\newcommand\numberthis{\addtocounter{equation}{1}\tag{\theequation}}

\begin{document}

\begin{frontmatter}
\title{Certified Data Removal Under High-dimensional Settings}
\author{Haolin Zou, Arnab Auddy, Yongchan Kwon, \\ Kamiar Rahnama Rad, Arian Maleki}

\runtitle{Certified Data Removal Under High-dimensional Settings}
\runauthor{Zou et al.}

\date{\today}

\begin{abstract}
Machine unlearning focuses on the computationally efficient removal of specific training data from trained models, ensuring that the influence of forgotten data is effectively eliminated without the need for full retraining. Despite advances in low-dimensional settings, where the number of parameters \( p \) is much smaller than the sample size \( n \), extending similar theoretical guarantees to high-dimensional regimes remains challenging. 


We propose an unlearning algorithm that starts from the original model parameters and performs a theory-guided sequence of Newton steps \( T  \in \{ 1,2\}\). After this update, carefully scaled isotropic Laplacian noise is added to the estimate to ensure that any (potential) residual influence of forget data is completely removed. 

We show that when both \( n, p \to \infty \) with a fixed ratio \( n/p \), significant theoretical and computational obstacles arise due to the interplay between the complexity of the model and the finite signal-to-noise ratio. Finally, we show that, unlike in low-dimensional settings, a single Newton step is insufficient for effective unlearning in high-dimensional problems---however, two steps are enough to achieve the desired certifiebility. We provide numerical experiments to support the certifiability and accuracy claims of this approach.

\end{abstract}

\end{frontmatter}

\section{Introduction}
Many real-world machine learning systems, including healthcare diagnostic tools and models like ChatGPT and DALL-E, rely on diverse user and entity data during training. If a user requests their data to be removed, it is reasonable to expect the responsible companies or entities to not only delete the data from their datasets but also eliminate its trace from the trained models. This process requires frequent and costly retraining of models. To address this challenge, the field of machine unlearning has emerged, focusing on efficient and less computationally intensive methods to remove dataset traces from models.

This field has made significant progress over the past few years \cite{cao2015towards, bourtoule2021machine, nguyen2022survey, chundawat2023zero, tarun2023fast, gupta2021adaptive, chen2021machine}. Substantial empirical research, coupled with rigorous theoretical results, have established a strong foundation for this area.

As we will clarify in Section \ref{ssec:detailedCOMP}, existing theoretical results in the field of machine unlearning usually have an implicit focus on low-dimensional settings, where the number of model parameters \( p \) is much smaller than the number of observations \( n \), i.e., \( p \ll n \). However, in many real-world applications, the number of parameters is comparable to—or even exceeds—the number of observations. This discrepancy raises a fundamental and currently unresolved question in the field:

\begin{center}
\textbf{Are existing machine unlearning methods reliable in high-dimensional regimes as well?}
\end{center}
The goal of this paper is to answer the above question for the machine unlearning algorithms that are based on the Newton method. More specifically, the paper aims to make the following contributions:

\begin{enumerate}

\item
We study the performance of machine unlearning algorithms in proportional high-dimensional asymptotic (PHAS) settings, where both the number of parameters \( p \) and the number of observations \( n \) are large, and their ratio \( n/p \to\gamma_0 \in(0,\infty) \).
\item 
As we will clarify later, some of the notions introduced for evaluating the certifiability and accuracy of machine unlearning algorithms — such as \(\epsilon\)-certifiability from \cite{guo2019certified} and the excess risk considered in \cite{sekhari2021remember} — are not well-suited to high-dimensional settings. To address this limitation, we refine some of these existing notions and propose new metrics tailored to evaluating the certifiability and accuracy of machine unlearning algorithms in high-dimensional regimes. 

\item  We consider the popular class of regularized empirical risk minimization (R-ERM), studied in previous works \cite{guo2019certified, sekhari2021remember, neel2021descent}, and analyze the performance of machine unlearning algorithms based on the Newton method under our high-dimensional setting. As a result, we show that:
\begin{enumerate}
\item Unlike the low-dimensional settings, machine unlearning algorithms based on a single Newton step (similar to those introduced in \cite{guo2019certified} and \cite{sekhari2021remember}) are not reliable, even when removing only one data point from the dataset.

\item In contrast, a machine unlearning algorithm yields a reliable estimate with two Newton iterations when $m=O(1)$. Furthermore, we quantify the minimum number of steps needed when $m$ grows with $n$, provided that $m=o(n^{\frac13})$.
\end{enumerate}

\end{enumerate}

\section{Our framework}\label{sec: framework}

\subsection{General framework of approximate machine unlearning}
\label{ssec: framework}

In this paper, we consider a \textit{generalized linear model} with an i.i.d. sample 
$\cD = \{ (y_1, \bx_1),$ $(y_2, \bx_2), \ldots, (y_n, \bx_n)\}\in \RR^{n\times (p+1)}$ as an independent and identically distributed (i.i.d.) sample from  some joint distribution
\[
    (y_i,\bx_i)\sim q(y_i | \bx_i^\top \bbeta^*) p(\bx_i),
\]
where $\bbeta^*\in\RR^p$ is the parameter of interest. An estimator of $\bbeta^*\in\Theta$, denoted as $\hb$, may be obtained from a \textit{learning algorithm} $A: \RR^{n\times (p+1)}\to\Theta$. To formalize the idea of data removal problem, let $\cM\subset\{1,2,...,n\}$ be the subset of data indices to be removed, let $\cD_{\cM}:=\{(y_i,\bx_i): i\in\cM\}$ be the corresponding subset of $\cD$, and let $\cDm:=\cD\backslash\cD_{\cM}$. In order to remove $\cD_{\cM}$ and all traces of it from the model, we essentially need to obtain $\hbm = A(\cDm)$ \footnote{Technically, $A$ refers to a sequence of functions $\{A_{n,p}:\RR^{n\times (p+1)}\to\Theta |\; n,p\in\mathbb{N}_+\}$ so $A(\cD) = A_{n,p}(\cD)$ and $A(\cDm) = A_{n-m,p}(\cDm)$ where $m:=|\cM|$, as the two functions are defined on different spaces thus cannot be identical. But we drop the subscripts for notational brevity.}, which is also called `exact machine unlearning'. However, in many applications, it is considered impractical to exactly compute $A(\cDm)$, so \textbf{approximate} removal methods are more desirable. Such methods calculate an efficient approximation $\tbm$ of $A(\cDm)$.

To evaluate $\tbm$, inspired by \cite{guo2019certified, dwork2006differential}, we consider the following two principles that a 'good' machine unlearning algorithm should satisfy:

\begin{enumerate}
    \item[P1.] [Certifiability] No information about the data points in \(\mathcal{M}\) should be recoverable from the unlearning algorithm, as the users have requested their data to be entirely removed.

    \item[P2.] [Accuracy] \(\tbm\) should be ``close" to \(\hbm\) in terms of down stream tasks such as out-of-sample prediction.
\end{enumerate}

The aforementioned principles serve as a conceptual framework for machine unlearning algorithms. However, to enable systematic evaluation and facilitate objective comparisons, it is essential that these principles be formalized into explicit, quantitative criteria and metrics.

Let us begin with the certifiability principle. Since \(\tbm\) is only an approximation of \(\hbm\), it is likely to inherently retain some information about $\cD_{\cM}$. To hide such residual information, random noise must be introduced into the estimate.\footnote{For more information about this claim, the reader can refer to the literature of differential privacy \cite{dwork2006differential}. } For example, independent and confidential noise can be directly added to \(\tbm\). Consequently, in studying the machine unlearning problem, we will focus on \textbf{randomized} approximations. We use the notation $\bb$ for a random perturbation used to obtain such randomized approximation. In addition, inspired by \cite{sekhari2021remember}, we also allow the system to store and use a summary statistic $T(\cD)$ for the unlearning algorithm, so it can be formalized as:
\[
    \tbm^{R} = \tA(\cD_{\cM},A(\cD), T(\cD),\bb).
\]

\noindent where the supersript $R$ in $\tbm^{R}$ stands for `randomized'. Inspired by the existing literature including \cite{guo2019certified,sekhari2021remember}, we propose the following criterion for the certifiability principle:
\begin{definition}[ $(\phi,\epsilon)$- Probabilistically certified approximate data removal (PAR)]
\label{def:epscert}
For $\phi,\eps>0$, a randomized unlearning algorithm $\tA$ is called a $(\phi,\epsilon)$-probabistically certified approximate data removal (PAR) algorithm, if and only if $\exists \mathcal{X}\subset \RR^{n\times (p+1)}$ with $\PP(\cD\in \mathcal{X})\geq 1-\phi$, such that $\forall \cD\in \mathcal{X}$ and $\forall \cM\subset [n]$ with $|\cM|\leq m$, $\forall$ measurable $\cT \subset \Theta$, we have

\begin{equation}\label{eq:epsilon_cert}
{\rm e}^{-\epsilon}<\frac{\mathbb{P} \left(\left. \tA(\cD_{\cM},A(\cD), T(\cD),\bb)  \in \cT \right| \cD\right)}{\mathbb{P}( \tA(\emptyset,A(\cDm), T(\cDm),\bb) \in \cT|\cD)} \leq {\rm e}^{\epsilon}
\end{equation}
where $\bb$ is a random perturbation used in the randomized algorithm $\tilde{A}$.
\end{definition}


There are a few points that we would like to clarify about this definition in the following remarks:

\begin{remark}\label{rmk: why adding phi}
Our criterion \ref{def:epscert} differs slightly from similar definitions in the literature, such as the ``$\epsilon$-Certified Removal'' in \cite{guo2019certified}, and the ``$(\epsilon,\delta)$-unlearning'' in \cite{sekhari2021remember}. A detailed comparison will be postponed to Section \ref{ssec: why_one_step_not_enough} and Section \ref{ssec:detailedCOMP}. The key motivation of such change is that the existing $\eps$ or $(\eps,\delta)$-certifiability guarantees require worst-case bounds on certain quantities (e.g. the gradient residual norm in Theorem 1 of \cite{guo2019certified}, and the global Lipschitzness constant of $\ell$ in Assumption 1 of \cite{sekhari2021remember}), which can be prohibitively large or even unbounded. Instead, under PHAS we obtain high-probability bounds for corresponding quantities (whose worst-case values are unbounded), e.g. in Lemma \ref{lem: direct_perturbation}. Following Lemma \ref{lem: direct_perturbation} there is a more detailed discussion over this claim.
\end{remark}

\begin{remark}
While the definition itself permits \(\phi\) to take any value, within our framework, we expect \(\phi\) to depend on $(m,n,p)$ and can be arbitrarily small for large enough $n,p$ and small enough $m$ (compared to $n$). This will be made more clear later in Theorem \ref{thm: main_epscert}.
\end{remark}

\begin{remark}
    The notation $\tA(\emptyset,A(\cDm), T(\cDm),\bb)$ in the denominator refers to the hypothetical outcome of the unlearning algorithm given the exact removal result $A(\cDm)$ and a null removal request. It should be interpreted as a \textbf{randomized version} of $A(\cDm)$, so that this criterion can be interpreted as the indistinguishability between the randomized unlearning algorithm and the randomized version of $A(\cDm)$ in distribution, which is similar as in \cite{guo2019certified}.\footnote{Technically we can make this more clear by defining a perturbation function $R(\bbeta,\bb)$ that returns a randomized version of $\bbeta$ using a random perturbation $\bb$, and define the non-randomized unlearning algorithm to be $\bar{A}(\cD_{\cM},A(\cD),T(\cD))$, and define $\tA:= \bar{A}\circ R$. If we assume $\bar{A}(\emptyset,\bbeta,T)\equiv \bbeta$, then we have $\tA(\emptyset,A(\cDm), T(\cDm),\bb) = R(A(\cDm))$, a randomized version of $A(\cDm)$. But for notational brevity, we do not introduce the notations above into the paper.}
\end{remark}
Intuitively, injecting a large amount of noise $\bb$ into the estimates can cover all residual information in $\tbm$ thus in favor of the certifiability principle and criterion \ref{def:epscert}, but it will harm the accuracy of it when used in downstream tasks such as prediction. To address this, we need a measure of accuracy:
\begin{definition}[Generalization Error Divengence (GED)]
\label{def: gen error div}
    Let $\ell(y|\bx^\top\bbeta)$ be a measure of error between $y$ and $\bx^\top\bbeta$, and let $(y_0|\bx_0)$ i.i.d. with the observations in $\cD$. Then the \textbf{Generalization Error Divengence (GED)}  of the learning and unlearning algorithms $A, \tA$ is defined as:
    \[
        {\rm GED}^{\epsilon}(A,\tA):=
        \EE\left(|\ell(y_0|\bx_0^\top A(\cDm)) - \ell(y_0|\bx_0^\top\tA(\cD_{\cM},A(\cD), T(\cD),\bb))|\big\vert\cD\right).
        \label{eq: def_err}\numberthis
    \]
\end{definition}

This metric measures the difference between the generalization error of the approximate  unlearing algorithm $\tA$ and the exact removal $A(\cDm)$, which naturally arises when we want to compare the prediction performance of $\tA$ on a new data point compared to exact removal $A(\cDm)$. Note that other notions have been introduced in the literature for measuring the accuracy of machine unlearning algorithms. For instance, \cite{sekhari2021remember} used the excess risk of $\tA$ against that of the true minimizer of population risk, but we will show in Section \ref{ssec:detailedCOMP} that such notions are not useful in the high dimensional settings.


%

\subsection{Regularized ERM and Proportional High-dimensional Asymptotic Settings}
\label{ssec: RERM and PHAS}

To estimate $\bbeta^*$ in the GLM model, researchers often use the following optimization problems known as regularized empirical risk minimization (R-ERM):\footnote{One simple extension of these ideas is to include more than one regularizer with mutiple regularization strengths $\lambda_1,\lambda_2,$ etc. For notational brevity we consider the simplest case, while generalization into multiple regularizers is straightforward. }
\vspace{-.3cm}
\begin{align*}
    &\hb = A(\cD) \triangleq  \underset{\bm{\beta} \in \RR^p}{\argmin}   \sum_{i=1}^n  \ell ( y_i|\bx_i^\top \bm{\beta} ) + \lambda r(\bm{\beta}) \label{eq:bl}\numberthis\\
    &\hbm = A(\cDm) \triangleq  \underset{\bm{\beta} \in \RR^p}{\argmin}   \sum_{j\notin \cM}  \ell ( y_j | \bm{x}_j^\top \bm{\beta} ) + \lambda r(\bm{\beta}) \label{eq:bli}\numberthis.
\end{align*}
In this optimization problem, $\ell( y|\bm{x}^\top \bm{\beta} )$ is called the loss function, which is typically set to $ - \log q(y| \bm{x}^\top \bm{\beta})$ when $q$ is known, and $r(\bbeta)$ is called the regularizer, which is usually a convex function minimized at $\bbeta=0$. It aims at reducing the variance of the estimate and therefore, the value of $\lambda \in [0,\infty)$ controls the amount of regularization. R-ERM is used in many classical and modern learning tasks, such as linear regression, matrix completion, poisson and multinomial regressions, classification, and robust principal components.


Again, the objective is to find an unlearning algorithm $\tA(\cD_{\cM},A(\cD), T(\cD),\bb)$ that is $(\phi,\epsilon)$-PAS (Definition \ref{def:epscert}) and has small ${\rm GED}^{\epsilon}$ (Definition \ref{def: gen error div}). 

As described before, we aim to study high-dimensional settings in which both (\(n\)) and (\(p\)) are large. Towards this goal, we use one of the most widely-adopted high-dimensional asymptotic frameworks, proportional high-dimensional asymptotic setting (PHAS).

The proportional high-dimensional asymptotic setting (PHAS) has provided valuable insights into the optimality and practical effectiveness of various estimators over the past decade \cite{MalekiThesis, DoMaMo09, MaMoCISS10, BaMo10, DoMaMoNSPT, BaMo11, mousavi2013asymptotic, el2013robust, donoho2016high, oymak2013squared, karoui2016can, AmLoMcTr13, KrMeSaSuZd12, KrMeSaSuZd12b, celentano2024correlation, miolane2021distribution, wang2022does, WangWengMaleki2020, li2021minimum, liang2022precise, dudeja2023universality, fan2022approximate, dobriban2018high, dobriban2019asymptotics, wainwright19HDS}.

\begin{definition}[Proportional High-dimensional Asymptotic Setting (PHAS)]\label{def:PHAS}
Assume that both  $n$ and $p$ grow to infinity, and that $n/p \rightarrow \gamma_0 \in (0,\infty)$. 
\end{definition}

While our theoretical goal is to derive finite-sample results applicable to any values of \(n\) and \(p\), PHAS (Definition \ref{def:PHAS}) will serve as a basis for simplifying and interpreting these results in high-dimensional settings. By default, all the following ``big O'' notations should be interpreted under the directional limit of PHAS, i.e. $n,p\to\infty, n/p\to \gamma_0\in(0,\infty)$. In contrast, in classical, or low dimensional settings, the direction of limit is usually $n\to\infty, p\equiv p_0$, in which case $n/p\to\infty$.

\subsection{Newton Method and Direct Perturbation}
\label{ssec: newton method}
Inspired by multiple algorithms in the literature of machine unlearning, such as the algorithms introduced in \cite{guo2019certified,sekhari2021remember},  in this paper we use the \textbf{Newton method} with \textbf{direct perturbation} to construct a  $(\phi, \epsilon)$-PAR algorithm with guaranteed prediction accuracy in terms of ${\rm GED}^{\epsilon}$. 

The Newton method, also called Newton-Raphson method, is essentially an iterative root-finding algorithm. 
\begin{definition}[Newton Method]
\label{def: Newton}
    Suppose $\bm{f}:\RR^p\to\RR^p$ has an invertible Jacobian matrix $\bG$ anywhere in an open set $\Theta\subset\RR^p$, and $\bm{f}$ has a root $\bm{f}(\bbeta^*)=\bzero$ in $\Theta$. Starting from an initial point $\bx^{(0)}\in\Theta$, the Newton method is the following iterative procedure: for step $t\geq 1$,
    \[
        \bx^{(t)} := \bx^{(t-1)} - \bG^{-1}(\bx^{(t-1)})\bbf(\bx^{(t-1)}).
    \]
\end{definition}

For more information about Newton method, please check Section 9.5 of \cite{BoydBook}.

Denote the objective function of $\hb$ as $L(\bbeta)$ in \eqref{eq:bl}, and that of $\hbm$ as $L_{\backslash \cM}$. Notice that finding $\hbm$ is equivalent to solving the root of the gradient of $L_{\backslash \cM}$ when it is smooth, and that $\hb$ is reasonably close to $\hbm$ if $m=|\cM|$ is small, so we can initialize the Newton method at $\tbm^{(0)}=\hb$ and iteratively compute:
\begin{align*}
    \tbm^{(t)}&= \tbm^{(t-1)} - \bGm^{-1}(\tbm^{(t-1)})\nabla L_{\backslash\cM}(\tbm^{(t-1)}),\quad t\geq 1
    \label{eq: def_t_step_newton}
\end{align*}
where $\bGm(\bbeta)$ is the Hessian of $L_{\backslash \cM}$, and $\nabla L_{\backslash \cM}$ is its gradient.

Suppose that we stop the Newton method after $T$ iterations. In order to obscure residual information of $\cD_{\cM}$, we introduce \textbf{direct perturbation}, resulting in our proposed unlearing method, \textbf{Perturbed Newton} estimator\footnote{Note that even though we do not distinguish the perturbations $\bb$, it should be independently drawn each time it is used.}:
\[
    \tbm^{R,T}= \tA(\cD_{\cM},A(\cD), T(\cD),\bb) \triangleq \tbm^{(T)} + \bb.
\]
Throughout this paper we consider Isotropic Laplacian distribution for $\bb$, which has the density
\[
    p_{\bb}(\bb)= \frac{C^p\Gamma(\frac{p}{2})}{2\pi^{\frac{p}{2}}\Gamma(p)}{\rm e}^{-C\Vert\bb\Vert}\propto {\rm e}^{-C\Vert\bb\Vert}.
\]
for some scale parameter $C>0$. Drawing a sample from $p_{\bb}(\bb)$ is equivalent to drawing $\Vert\bb\Vert$ from a $Gamma(p,C^{-1})$ distribution, then sampling $\bb$ uniformly on the sphere with radius $\Vert\bb\Vert$ (see Lemma \ref{lem: laplace and gamma}). The reason to consider this distribution is that its log density is Lipschitz in $\|\bb\|$, thus naturally connected with the $(\phi,\epsilon)$-PAR criteria. This will be shown later in more details in Lemma \ref{lem: direct_perturbation}.

It remains a question when to stop the Newton iteration. To find an `exact' solution, it is usually run until certain convergence criterion is met. In contrast, it has been proposed in the literature that one Newton step suffices in the low dimensions, for example in \cite{guo2019certified, sekhari2021remember}. However, as will be clarified later, under PHAS even to remove a single data point ($m=1$), we need \textbf{at least two Newton steps} to guarantee good prediction accuracy (i.e. ${\rm GED}^\epsilon\to 0$). This will be discussed later in Section \ref{ssec: main}.

\subsection{Notations}  
We adopt the following mathematical conventions. Scalars and scalar-valued functions are denoted by  English or Greek letters (e.g. $C_1, a, \lambda>0$)\footnote{Usually we use $C(n)$ to indicate a quantity that grows at most at a speed of $\polylog(n)$,  otherwise we tend to put $n$ in the subscript.}. Caligraphic uppercase letters are used for sets, families or events (e.g. $\cT\subset \RR^p$) with an exception that $\cN$ refers to the Gaussian distribution. $\RR,\RR_+$ denote the set of real, positive real numbers respectively. Vectors are represented by bold lowercase letters (e.g. $\bx\in\RR^p$), and matrices by bold uppercase (e.g. $\bX\in\RR^{n\times p}$). For a matrix $\bX$, $\Vert\bX\Vert$, $\Vert\bX\Vert_{Fr}$, $\lambda_{\min}(\bX)$, $\lambda_{\max}(\bX)$, $\sigma_{\min}(\bX)$, $\sigma_{\max}(\bX)$ and ${\rm tr}(\bX)$ denote the (Euclidean) operator norm, Frobenius norm, minimal and maximal eigenvalues, minimal and maximal singular values and the trace of $\bX$, respectively. Moreover,
\[
    \Vert\bX\Vert_{p,q}:=\sup_{\Vert\bw\Vert_p\leq 1}\Vert\bX \Vert_q
\]
is the operator norm induced by  $\ell_p, \ell_q$ norms, for $p,q\in \RR_+\cup\{+\infty\}$.

We denote $[n]:=\{1,2,\cdots,n\}$ for some $n\in\mathbb{N}_+$. For any index set \(\mathcal{M}\subset[n]\), we use \(\bm{X}_{\mathcal{M}}\) to denote the sub-matrix of \(\bm{X}\) that consists of the rows indexed by \(\mathcal{M}\). Similarly, \(\bm{a}_{\mathcal{M}}\) represents the sub-vector of \(\bm{a}\) containing the elements indexed by \(\mathcal{M}\).  More generally, we use the subscript $\cdot_\cM$ to refer to a quantity corresponding to $\cD_{\cM} = \{(y_i,\bx_i), i\in\cM\}$, and the subscript $\cdot_{\backslash\cM}$ to refer to the quantities corresponding to $\cD_{\backslash\cM} := \cD\backslash\cD_{\cM}$.

We define \(\ld(y|z)\) and \(\ldd(y|z)\) as the first and second derivatives of the function \(\ell\) with respect to \(z\), respectively. Furthermore, we write $\ell_i(\bbeta)$ for short of $\ell(y_i|\bx_i^\top\bbeta)$, and similarly $\ld_i(\bbeta)$ and $\ldd_i(\bbeta)$. Additionally, we introduce the vectors  

\[
\bm{\ld} := \left[\ld_1(\hb), \cdots, \ld_n(\hb)\right]^\top, \quad
\bm{\ldd} := \left[\ldd_1(\hb), \cdots, \ldd_n(\hb)\right]^\top.
\]  

We define \(\diag[\bm{a}]\) or $\diag[a_i]_{i\in[n]}$ as a diagonal matrix whose diagonal elements correspond to the entries of the vector \(\bm{a} = (a_1,\cdots,a_n)^\top\).

We use the following notations for the limiting behavior of sequences. We use $\polylog(n)$ as a shorthand for finite degree polynomials of $\log(n)$. We use the conventional notations for limiting behavior of sequences: $a_n = o(b_n)$, $O(b_n)$, $\Omega(b_n)$, $\Theta(b_n)$ respectively mean that $a_n/b_n$ is convergent (to 0), bounded, divegent, and asymptotically equivalent. We use similar notations for their stochastic analogies, e.g. $X_n = O_p(1)$ iff $X_n\to 0$ in probability, and so forth. Finally, the symbol ``$\independent$'' means ``independent'' in probability.

\section{Our contributions}

In this section, we present our main theoretical results on the certifiability and accuracy of the machine unlearning algorithms introduced in Section~\ref{ssec: newton method}, under the proportional asymptotic regime. Before stating the results, we first provide a detailed overview of the assumptions underlying our analysis.

\subsection{Main assumptions}
\label{ssec:assumptions}
Our first series of assumptions are concerned with the structural properties of $\ell$ and $r$. 

\begin{assumptionA}[Separability]\label{assum:separability}
    The regularizer is separable:
    \[
        r(\bbeta) = \sum_{k\in[p]}r_k(\beta_k).
    \]
\end{assumptionA}

This assumption can be generalized to include a linear transform:
$r(\bbeta) = \sum_{j\in[l]} r_j(\ba_j^\top \bbeta)$, but we make the simplified assumption that $\ba_j = \bm{e}_j$ where $\bm{e}_j$ is the jth canonical basis of $\RR^p$ to avoid cumbersome notations. Generalizing the current proof to arbitrary $\ba_j$ is trivial.

\begin{assumptionA}[Smoothness]\label{assum:smoothness}
Both the loss function $\ell:\RR\times\RR\to\RR_+$ and the regularizer $r:\RR^p\to\RR_+$ are twice differentiable. 
\end{assumptionA}

\begin{assumptionA}[Convexity]\label{assum:convexity}
Both $\ell$ and $r$ are proper convex, and $r$ is $\nu$-strongly convex in $\bbeta$ for some constant $\nu>0$.
\end{assumptionA}

These assumptions ensure that the R-ERM estimators $\hb$ and $\hbm$ are unique, and the Newton method is applicable. While true for many applications, in other cases where certain structures such as sparsity of $\beta$ is assumed, these assumptions can be violated. While a few papers have shown how the Newton method can be extended to non-differentiable settings (e.g.~\cite{auddy24a,wang2018approximate}), the theoretical study of such cases will not be the focus of this work, and are left for a future research. Note that we implicitly assumed $\ell$ and $r$ to be non-negative without loss of generality. This can be achieved by subtracting their minimum (which are finite since $\ell,r$ are proper convex) from the function themselves.

In addition to the structural properties, we make several assumptions on the probablisic aspects of the data and model:

\begin{assumptionB}\label{assum:normality}
The feature vectors $\bm{x}_i \overset{iid}{\sim} \mathcal{N} (\bzero, \bm{\Sigma})$. Furthermore, we assume that $ \lambda_{\max}(\bSigma) \leq \frac{C_X}{p}$, for some constant $C_X>0$. 
\end{assumptionB}


The Gaussianity assumption is prevalent in theoretical papers dealing with high-dimensional problems, for example \cite{miolane2021distribution,WengMalekiZheng18,rad2018scalable,auddy24a}. Although our proofs can be generalized to a broader class of distributions of $\bx_i$ beyond Gaussianity, we do not discuss it in details in this paper.

The scaling we have adopted in the above assumption is based on the following rationale. First notice that since $\bx_i \sim \cN(0,\bSigma)$,
\[
{\rm var}(\bm{x}_i^\top \bm{\beta}^*)=\mathbb{E} (\bm{x}_i^\top \bm{\beta}^*)^2 \leq \frac{C_X}{p} \|\bm{\beta}^*\|_2^2. 
\]

Heuristically speaking, uner PHAS and when the elements of $\bm{\beta}^*$ are $O(1)$,  we have $\|\bm{\beta}^*\|_2=O(\sqrt{p})$, and hence $\mathbb{E} (\bm{x}_i^\top \bm{\beta}^*)^2 = O(1)$. On the other hand, it is reasonable to assume that $y_i|\bx_i^\top\bbeta$ has $\Theta(1)$ variance. Therefore, under the settings of the paper we can see that the signal-to-noise ratio (SNR) of each data point, defined as $\frac{{\rm var}(\bm{x}_i^\top\bm{\beta}^*)}{{\rm var}(y_i|\bm{x}_i^\top\bm{\beta}^*)}$, remains bounded. We now introduce two more assumptions on the likelihood function $\ell$ and the response $y$. These are typically used in the analysis of high dimensional regression problems and are satisfied for a host of natural examples including linear and logistic regression. See, e.g., \cite{zou2024theoretical}.

\begin{assumptionB}\label{assum:ld}
    $\exists C,s>0$ such that 
    \[
        \max\{\ell(y,z), |\ld(y,z)|,|\lddd(y,z)| \}\leq C(1+|y|^s + |z|^s)
    \]
    and that $\nabla^2 r(\bbeta) = \diag[\rdd_k(\beta_k)]_{k\in[p]}$ is $C_{rr}(n)$-Lipschitz (in Frobenius norm) in $\bbeta$ for some $C_{rr}(n)=O(\polylog(n))$.
\end{assumptionB}

This assumption requires that the derivatives of $\ell$ grows with $y$ and $z$ at most as fast as a polynomial function with order $s$, and the regularizer should be $O(\polylog(n))$-Hessian-Lipschiz, with one example be the ridge penalty: $r(\bbeta) = \|\bbeta \|^2$ and $\nabla^2r(\bbeta)=\II_p$.

\begin{assumptionB}\label{assum:y}
    $\PP(|y_i|>C_y(n))\leq q_y(n)$ and $\EE|y_i|^{2s}\le C_{y,s}$ for some $C_y(n)=O(\polylog(n))$, a constant $C_{y,s}$ and $q_n^{(y)}=o(n^{-1})$
\end{assumptionB}
This assmption essentially requires all $|y_i|$ to be stochastically bounded even when $n,p$ increases.

Below are some examples where Assumptions \ref{assum:ld} and \ref{assum:y} are satisfied. For simplicity we assume $\bx_i\sim \cN(0,\frac1p \II_p)$.

\begin{example}[Linear regression]
    Suppose $y_i|\bx_i\sim \cN(\bx_i^\top\bbeta^*,\sigma^2)$, then we have $y_i\sim\cN(0,\tau^2)$ with $\tau^2:=\sigma^2 + \frac1p \Vert\bbeta^* \Vert^2$. Its negative log-likelihood is the $\ell_2$ loss:
    \begin{align*}
        \ell(y,z) = \frac12 (y-z)^2, \quad
        \ld(y,z) = z-y ,\quad
        \ldd(y,z)  = 1,\quad
        \lddd(y,z)  = 0.
    \end{align*}
    And by Lemma \ref{lem: conc_single_x} we have the following concentration for $y_i$:
    \[
        \PP(|y_i|\geq 6\tau\sqrt{\log(n)})\leq \frac{1}{\sqrt{6\pi}}n^{-3}.
    \]
\end{example}

\begin{example}[Logistic regression]
    Suppose $y_i\sim Bernoulli(p_i)$ where $p_i = (1+e^{-\bx_i^\top\bbeta^*})^{-1}$. The negative log-likelihood is then
    \begin{align*}
        \ell(y,z) &= y \log(1+e^{-z}) + (1-y)\log(1+e^z), \quad y \in\{0,1\}\leq 2\log(2)+2z,\\
        |\ld(y,z)| &= \left|\frac{e^z}{1+e^z} - y\right|\leq 1+|y|,\\
        |\ldd(y,z)| &= \left| \frac{e^z}{(1+e^z)^2} \right|\leq 1 \\ 
        |\lddd(y,z)| &=\left| -\frac{1}{1+e^z}+\frac{3}{(1+e^z)^2}-\frac{2}{(1+e^z)^3}\right| \leq 6,
    \end{align*}
    and obviously $|y_i|\leq 1$ so that Assumption \ref{assum:y} is also satisfied.
\end{example}

\subsection{Main theorem and its implications}
\label{ssec: main}

The main objective of this paper is to answer the following two questions: 

\begin{enumerate}
    \item[$\mathcal{Q}_1$:] Given a $t$-step Newton estimator $\tbm^{(t)}$, can we find a large enough perturbation $\bb$ so that $\tbm^{R,t}$ is $(\phi,\epsilon)$-PAR, for some $\phi\to 0$ under PHAS?
    \item[$\mathcal{Q}_2$:] Given the perturbation level in $\mathcal{Q}_1$, can we find a sufficient number of Newton steps $T$ such that ,
    ${\rm GED}^{\epsilon}(\hbm, \tbm^{R,T})\to 0$ under PHAS?
\end{enumerate}

The two theorems below are the main results of our paper. They guarantee the certifiability and accuracy of the Perturbed Newton estimator and answer $\mathcal{Q}_1$ and $\mathcal{Q}_2$ respectively.
\begin{theorem}\label{thm: main_epscert}
    Under Assumptions A1-A3 and B1-B3, suppose $m=o(n^{\frac13})$,  suppose $\bb$ has density $p_{\bb}(\bb)\propto{\rm e}^{-\frac{\epsilon}{r_{t,n}}\Vert\bb\Vert}$ with
    \[
        r_{t,n} = [C_1(n)]^{2^{t-1}}\left(\frac{C_2(n)m^3}{2\lambda\nu n}\right)^{2^{t-2}},
    \]
    for some $C_1(n),C_2(n)=O(\polylog(n))$ and $\epsilon>0$.
    Then $\tbm^{R,t} = \tbm^{(t)}+\bb$ achieves $(\phi_n,\epsilon)$-PAR with 
    \[
        \phi_n = nq_n^{(y)} + 8n^{1-c} + ne^{-p/2} + 2e^{-p}\to 0
    \]
\end{theorem}
The proof, including the explicit expressions for $C_1(n), C_2(n)$, can be found in Section \ref{ssec: proof_main}. Theorem \ref{thm: main_epscert} shows that, with a certain noise level we can obtain a $(\phi_n, \epsilon)$-PAR algorithm from any steps of Newton iterations,  However, it does not provide information on the accuracy of the approximations. Recall the metric of accuracy we defined:
\[
    {\rm GED}^\epsilon(\tbm^{R,t},\hbm) = 
    \EE\left(\big\vert
    \ell(y_0, \bx_0^\top(\hbm)) - \ell(y_0, \bx_0^\top(\tbm^{R,t}))
    \big\vert 
    |\cD\right).
\]
Our next theorem calculates the accuracy of the estimates that are obtained from the Newton method. 
\begin{theorem}
    \label{thm: main_accuracy}
    Under Assumptions A1-A3 and B1-B3, with probability at least $1-(n+1)q_n^{(y)}-14n^{1-c}-ne^{-p/2}-2e^{-p}-e^{-(1-\log(2))p}$, 
    \begin{align*}
        {\rm GED}^\epsilon(\tbm^{R,t},\hbm)
        &\leq \left(\frac{2\sqrt{p}}{\epsilon}+\frac{1}{\sqrt{p}}\right)r_{t,n}\sqrt{2m+2s}\cdot\polylog(n), \quad \forall t\geq 1.
    \end{align*}
    where $s$ is the constant in Assumption~\ref{assum:y}.
    Moreover, let $\alpha:=\log(m+1)/\log(n).$ If $t>T=1+\log_2\left(\frac{\alpha+1}{1-3\alpha}\right)$, then under PHAS, for any $\cD\in F^c$, we have
    \[
        {\rm GED}^\epsilon(\tbm^{R,t},\hbm)=O_p(1).
    \]
\end{theorem}

The proof of Theorem \ref{thm: main_accuracy} can be found in Section \ref{ssec: proof_main}.

Theorems \ref{thm: main_epscert} and \ref{thm: main_accuracy} answer the two questions $\mathcal{Q}_1$ and $\mathcal{Q}_2$ we raised at the beginning of this section:

\begin{enumerate}
    \item[$\cA_1$]: For any $t\geq 1$, the perturbation scale $r_{t,n}$ should be at least 
    $O\left((m^3/n)^{2^{t-2}}\right)$ so that $t$ steps of Newton step is $(\phi,\epsilon)$-PAR with $\phi\to 0$ under PHAS.

    \item[$\cA_2$]: The number of iterations $t$ should satisfy 
    \[
    t>T=1+\log_2\left(\frac{1+\alpha}{1-3\alpha}\right)\]
    where $\alpha = \log(m+1)/\log(n)$, so that ${\rm GED}^\epsilon(\tbm^{R,t},\hbm)\to 0$ in probability under PHAS.
\end{enumerate}

Note that in $\cA_2$, $\alpha=\log(m+1)/\log(n)>0$ for $m\ge 1$, so $t>1+\log_2(1)=1$. It again verifies our claim at the end of Section \ref{ssec: newton method} that \textbf{one Newton step is not enough}, even with $m=1$. However, when $m=O(1)$, $\alpha$ can be arbitrarily small when $n$ is large, in which case $\log_2(\frac{1+\alpha}{1-3\alpha})<1$ so $t=2$ is enough. For $m$ increasing with $n$, $\cA_2$ provides the minimum number of iterations needed, provided that $m=o(n^{\frac13})$.

Regarding the sharpness of Theorem \ref{thm: main_accuracy}, Figure \ref{fig:onetwo} illustrates that the amount of noise required for certified unlearning with a single Newton step is too large—it not only erases the targeted information but also corrupts parts of the model that should remain intact. In contrast, with two Newton steps, the required noise is significantly smaller, enabling removal of the intended information while preserving the rest of the model.

\begin{figure}[htbp]
    \centering
    \includegraphics[width=0.8\textwidth]{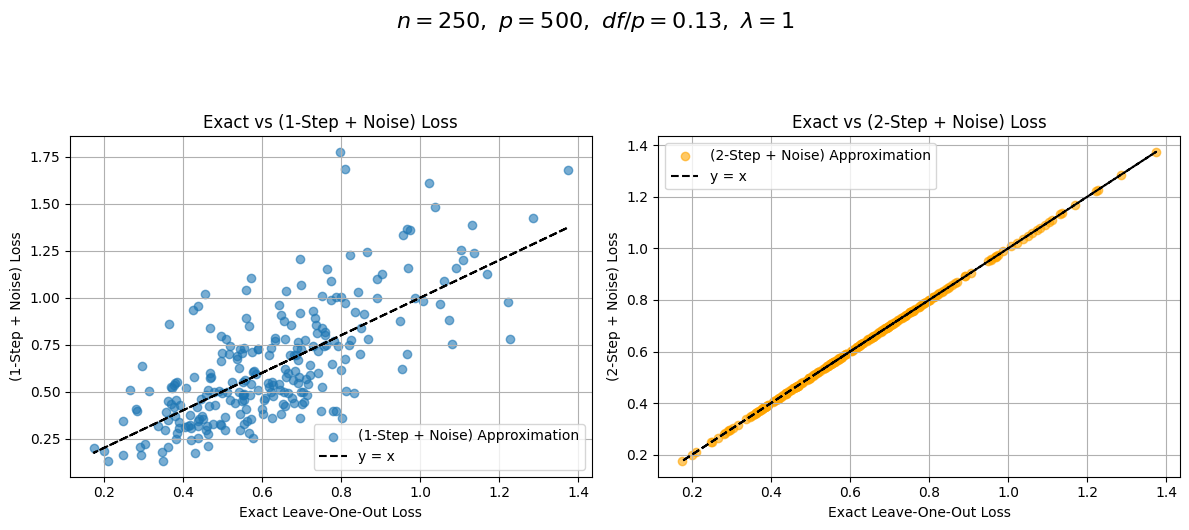} 
    \caption{The impact of $T$ Newton iterations on the accuracy of certified data removal for $T=1,2$. On both plots, the $X$ axis denotes the exact leave-one-out loss. Then, the figure on the left plots the one Newton step plus noise loss on the $Y$ axis. The figure on the right plots the two Newton step plus noise loss on the $Y$ axis. }
    \label{fig:onetwo}
\end{figure}


\subsection{Why Single Newton Step is Insufficient}
\label{ssec: why_one_step_not_enough}

In this section we explain in more details why one Newton step is not enough in high dimensions, by providing some key points in the proof of Theorem \ref{thm: main_epscert} and Theorem \ref{thm: main_accuracy},which also provide evidence for the sharpness of the two theorems. 

We start by explaining the trade-off between certifiability and accuracy in Section \ref{sssec: trade-off}.  It eventually turned out that with isotropic Laplace perturbation, the key to both criteria (and their sharpness) is the quantity
\[
    \|\tbm^{(t)}-\hbm \|_2,
\]
where $\tbm^{(t)}$ is the t-step Newton estimator before perturbation.
We then provide bounds for this quantity in Section \ref{sssec: l2_error_newton}, while detailed proofs are postponed to the Appendix.
\subsubsection{Trade-Off between Certifiability and Accuracy}
\label{sssec: trade-off}

Recall that we proposed two principles, namely certifiability and accuracy, and quantified them by $(\phi,\epsilon)$-PAR (Definition \ref{def:epscert}) and the Generalization Error Divergence (GED, Equation \ref{eq: def_err}) . However, as the title suggests, they usually do not work in the same direction: a method in favor of certifiability tends to use a large perturbation to obscure residual information, which usually sacrifices its accuracy, and vice versa. In this section, we will provide a more specific discussion on this trade-off.

We mentioned that the $(\phi,\epsilon)$-PAR and Isotropic Laplace perturbation are naturally connected, as the following lemma suggests:

\begin{lemma}\label{lem: direct_perturbation}
    Let $\hbm, \tbm\in \RR^p$ be any two estimators calculated from $\cD$. Suppose $\bb\in\RR^p$ is a random vector independent of $\cD$ and has a density $p_{\bb}(\bb)\propto{\rm e}^{-\frac{\epsilon}{r}\Vert\bb\Vert}$. Define 
    \[
    \cX_r\triangleq\{\cD: \underset{|\cM|\leq m}{\max}\Vert\hbm-\tbm\Vert_2\leq r\},
    \] 
    then $\forall$ measurable $ \cT\subset \RR^p$,  $\forall |\cM|\leq m$,
    \[
        {\rm e}^{-\epsilon}<\frac{\mathbb{P} \left(\tbm+\bb  \in \cT | \cD\right)}{\mathbb{P}(\hbm+\bb \in \cT|\cD)} \leq {\rm e}^{\epsilon},
        \numberthis\label{eq: direct_perturbation_LR}
    \]
    if and only if $\cD\in \cX_r$.
\end{lemma}
The proof can be found in Section \ref{ssec: proof_lem_direct_perturbation}.

The Lemma provides a necessary and sufficient condition for \eqref{eq: direct_perturbation_LR} to hold: the set $\cX_r$ precisely characterizes the kind of dataset $\cD$ for which the directly perturbed estimator $\tbm^R:=\tbm+\bb$ satisfies \eqref{eq: direct_perturbation_LR}. Therefore, if we define $\phi_r:=\PP(\cD\in\cX_r^c)$, then $\tbm^R$ satisfies $(\phi_r,\epsilon)$-PAR. Since $\phi_r$ is non-increasing in $r$, a larger $r$ is more desirable. However, a larger $r$ generally results in higher ${\rm GED}^\epsilon$ and lower prediction accuracy compared to $\hbm$. 

Note that the $\epsilon$-certified removal criterion in \cite{guo2019certified} corresponds to $\phi=0$ case, which requires we set 
\[r = r_{\max}:=\underset{\cD\in\RR^{n\times(p+1)}}{\sup}\underset{|\cM|\leq m}{\max}\Vert\tbm-\hbm\Vert.\] 
But the supremum may converge to $+\infty$  for many models when $n,p\to\infty$ under PHAS, rendering the original $\epsilon$-CR conceptually infeasible, especially when the loss function and the observations have unbounded supports. 

Even if $r_{\max}$ is finite, we may not want to set $r$ to this value, because a large $r$ could deteriorate the prediction accuracy of the model. In fact, by Lemma \ref{lem: laplace and gamma}, we have $\Vert\bb \Vert_2\sim Gamma(p,\frac{\epsilon}{r})$ with $\Vert\bb\Vert = o_p(\frac{2pr}{\epsilon})$. Suppose $\ell$ is smooth, then by Taylor expansion,
\begin{align*}
   {\rm GED}^{\epsilon} (\hbm, \tbm^R) 
   &\triangleq | \ell ( y_0|\bx_0^\top \hbm )  -   \ell ( y_0|\bx_0^\top (\tbm+\bb) )|\\
   &=|\ld_0(\bxi)\bx_0^\top(\hbm-\tbm-\bb)|
\end{align*}
where $\bxi = t\hbm + (1-t)(\tbm+\bb)$ for some $t\in[0,1]$.
In many cases, it can be shown that $|\ld_0(\bxi)|=\Theta_p(1)$. Suppose $\bx_0\sim \cN(\bzero,\frac1n \II_p)$ (this is a special case of Assumption \ref{assum:normality} in Section \ref{ssec:assumptions}, with justifications therein), then conditional on $\cD$ and $\bb$,
\begin{align*}
   |\bx_0^\top(\hbm-\tbm-\bb)|
   &=\Theta_p\left(\frac{1}{\sqrt{n}}\Vert \hbm-\tbm-\bb\Vert\right)\\
  (\text{since }\bb \independent \cD)\quad &=\Theta_p\left(\frac{1}{\sqrt{n}}\Vert \hbm-\tbm\Vert + \frac{1}{\sqrt{n}}\Vert\bb\Vert\right)\\
   &=\Theta_p\left(\frac{1}{\sqrt{n}}
   \left(1+\frac{p}{\epsilon}\right)r\right).
\end{align*}
According to the heuristic calculation above, in order that ${\rm GED}^{\epsilon} (\tbm+\bb)\to 0$ under PHAS, we need $r=o(\sqrt{n}(1+\frac{p}{\epsilon})^{-1})$. This is the second reason we want to use a stochastic bound in the definition of $(\phi,\epsilon)$-PAR, instead of a worst-case bound in the original $\epsilon$-CR definiton: \textbf{the stochastic bound can be much smaller than the worst case bound in high dimensions.} Consequently, we can reduce prediction error significantly with a small sacrific of $\phi$, which can also be arbitrarily small by our theory that will be discussed later.

From the discussions above, we know that the sharpness of Theorem \ref{thm: main_epscert} and \ref{thm: main_accuracy} depends on the sharpness of the bound we can obtain for $\| \tbm-\hbm\|_2$.

\subsubsection{The $\ell_2$ error of Newton Estimators}
\label{sssec: l2_error_newton}
We will start from $T=1$ case, and the goal of this section is to provide a stochastic bound for the $l_2$ error of one-step Newton estimator $\Vert\tbm^{(1)}-\hbm\Vert_2$ under PHAS. Recall that in the definition of $(\phi,\epsilon)$-PAR (\ref{def:epscert}), a set $\cX_r$ was introduced to represent the `good event' on which $\Vert\tbm^{(1)}-\hbm\Vert_2$ is small:
\[
    \cX_r^{(t)}:=\{\cD: \Vert\tbm^{(t)}-\hbm\Vert_2\leq r\},
\]
where we added an additional superscript $(t)$ to indicate the number of iterations. Note that
$r$ should be chosen such that $\phi_r=\PP(\cD\notin\cX_r^{(t)})$ is small.

\begin{theorem}\label{thm: one-step-new}
    Under Assumptions A1-A3 and B1-B3
    \[
        \PP(\cD\in\cX_r^{(1)})\geq 1-nq_n^{(y)} - 8n^{1-c} - ne^{-p/2} - 2e^{-p}
    \]
    for some $r=\frac{m^{\frac32}}{\sqrt{n}}\polylog(n)$.
\end{theorem}
\begin{remark}
    The exact form of $r$ can be found in Lemma \ref{lem:l2-norm-diff}.
\end{remark}
This theorem provides a theoretical explanation of Figure \ref{fig:p_scaling_npequal} left panel, which indeed shows that $\Vert\hbm-\tbm^{(1)} \Vert_2 = O_p(\frac{1}{\sqrt{n}})$ for logistic regression with ridge penalty with $m\equiv 1$. However, we acknowledge that our rate $\frac32$ of $m$ might not be sharp. In fact, figure \ref{fig:m_scaling} suggests that $\Vert\tbm^{(1)}-\hbm\Vert_2 = O_p(\frac{m}{\sqrt{n}})$.

Next we provide a bound the $\ell_2$ error of t-step Newton estimator:

\begin{theorem}
\label{thm: l2_error_multistep}
    Under Assumptions A1-A3, B1-B3, and suppose $m=o(n^{1/3})$, then 
    \[
        \PP(\cD\in\cap_{t=1}^\infty\cX_{r_{n,t}}) \geq 1-nq_n^{(y)} - 8n^{1-c} - ne^{-p/2} - 2e^{-p}
    \]
    for some $r_{n,t}=\left(\frac{m^3}{n}\right)^{2^{t-2}}\polylog(n)$.
\end{theorem}
The proof can be found in Section \ref{sssec: multi-step} in the appendix, where we also provide the exact form of the $\polylog(n)$ term. 

This theorem not only provides a stochastic bound for t-step Newton estimators, but provides a simutaneous bound for all $\Vert \tbm^{(t)}-\hbm\Vert_2$: when $m=o(n^{1/3})$, it implies that 
\[
    \Vert \tbm^{(t)}-\hbm\Vert_2=o_p\left(\left(\frac{m^3}{n}\right)^{2^{t-2}}\polylog(n)\right),   \quad \forall t\geq 1.
\]

The left panel of Figure \ref{fig:p_scaling_npequal} and Figure \ref{fig:m_scaling} together suggests that $\Vert\tbm^{(1)}-\hbm\Vert_2\sim \Theta(\frac{m}{\sqrt{n}})$ for ridge logistic model, so our result in Theorem \ref{thm: one-step-new} is sharp in $n$ but not necessarily in $m$. Nonetheless, $\Vert\tbm^{(1)}-\hbm\Vert_2$ converges no faster than $O(\frac{1}{\sqrt{n}})$, so according to the heuristic arguments in the end of Section \ref{sssec: trade-off}, our conclusion--- a single Newton step is insufficient in high dimensions--- remains valid even for $m=1$ case.

The middle panel of Figure \ref{fig:p_scaling_npequal} also appears that $\Vert\tbm^{(2)}-\hbm\Vert_2\sim (\frac{m^2}{n^{3/2}})$, so whether our result in Theorem \ref{thm: l2_error_multistep} is sharp remains an interesting research problem for future research.

\section{Related Work}
\label{sec: related work}

\subsection{Summary of the existing results}
\label{ssec: existing_results}
As we discussed earlier, the machine unlearning problem has received significant attention in recent years, both from theoretical and empirical perspectives \cite{nguyen2022survey}, \cite{suriyakumar2022algorithms}. 
Among the existing work, \cite{guo2019certified,sekhari2021remember,neel2021descent, izzo21a} 
are most closely related to our contributions, as they focus on theoretical aspects of the machine unlearning algorithms. Therefore, we provide a more detailed comparison between their contributions and ours.

In \cite{guo2019certified}, the authors introduced the concepts of $\epsilon$-certified and $(\epsilon, \delta)$-certified data removal. Our notion of \((\phi, \epsilon)\)-probabilistically certified approximate data removal, introduced in Definition~\ref{def:epscert}, is inspired by the \(\epsilon\)-certified data removal notion of \cite{guo2019certified}. However, our definition is more flexible, allowing the unlearning algorithm to be non-private on datasets that occur with low probability.  Furthermore, \cite{guo2019certified} analyzed the level of Laplacian noise that must be added to the objective function to ensure that the output of a single-step Newton method satisfies either \(\epsilon\)-certifiability or \((\epsilon, \delta)\)-certifiability. In these studies, the authors assumed that $n$ is large, while $p$ is fixed.

The work of \cite{sekhari2021remember} builds upon and improves the results of \cite{guo2019certified} in several directions: (1) they incorporate the dependence on the number of parameters or features \( p \) in their analysis, and (2) they introduce a notion of excess risk to evaluate the accuracy of the approximations produced by machine unlearning algorithms. Their main conclusion is that a single Newton step suffices to yield an accurate machine unlearning algorithm— a conclusion that stands in contrast to the message of our paper. We argue that the analysis presented in \cite{sekhari2021remember} lacks sharpness, and as a result, the bounds they derive are not useful in many high-dimensional settings. In fact, under the high-dimensional regime considered in our work, many of the bounds in \cite{sekhari2021remember} diverge as \( n, p \to \infty \). Since a thorough clarification of this point requires several pages, we defer the detailed discussion to Section~\ref{ssec:detailedCOMP}.

The authors of \cite{neel2021descent} have studied gradient-based methods initialized with the pre-trained models for machine unlearning, and established their theoretical performance—particularly in terms of the number of gradient descent iterations required. The discussions we present in Section~\ref{ssec:detailedCOMP} can be used for interpreting the results of \cite{neel2021descent} in high-dimensional settings as well. 

The authors of \cite{izzo21a} proposed a projection-based update method called projective residual update (PRU),  applicable to linear and logistic regression. It reduces the general $O(mp^2)$ time complexity of one Newton step to $O(m^2p)$, as it considers only the projection onto the $m-$dimensional subspace spanned by $\bX_{\cM}$. However, it lacks performance guarantee for more general models even in the low dimensional settings.


In parallel with the theoretical advances in machine unlearning, many empirical methods have also been studied, particularly for deep neural network models. A standard approach is to perform the gradient ascent algorithm on a forget set or the gradient descent algorithm on a remaining set \citep{graves2021amnesiac, goel2022towards}. 
While standard methods typically rely on either a forget set or a remaining set, \citet{kurmanji2023towards} proposed a novel loss function that leverages both datasets. Specifically, they proposed a new loss function that encourages an unlearned model to remain similar to the original on the remaining data, while diverging on the forget set. As an alternative approach, \citet{foster2024fast} proposed a training-free machine unlearning algorithm, demonstrating solid performance at limited computational cost. Compared to our work, all these aforementioned approaches have shown promising empirical results, but they rely on heuristics and lack theoretical guarantees on data removal. Along these lines, \citet{pawelczyk2024machine} demonstrated that many available machine unlearning algorithms are not effective in removing the effect of poisoned data points across various settings, which calls for many principled research works in this field. The readers may refer to \cite{xu24survey,li25survey} for a comprehensive survey of machine unlearning.

\subsection{Detailed comparison with \cite{sekhari2021remember}}\label{ssec:detailedCOMP}

To clarify some of the subtleties that influence the analysis in \cite{sekhari2021remember}, we revisit the assumptions and results of that work within the context of the setting and assumptions outlined in Section~\ref{ssec:assumptions} of our paper. To make our problem similar to the one studied in  \cite{sekhari2021remember}, define:
\[
f(\bbeta; \bx, y) = \ell (y | \bx^{\top} \bbeta) + \frac{\lambda}{n} r(\bbeta). 
\]
We further define
\[
\hat{F}_n (\beta) = \frac{1}{n} \sum_{i=1}^n f(\bbeta; \bx_i, y_i), 
\]
and 
\[
F(\bbeta) = \mathbb{E} \ell(y | \bx^{\top} \bbeta), 
\]
where the expected value is with respect to $(y, \bx)$.
Assumption (1) of \cite{sekhari2021remember} states that: \\

\noindent\textbf{Assumption 1 of \cite{sekhari2021remember}}. For any $(y,\bm{x})$, $f(\bbeta; \bm{x},y)$ as a function of $\bbeta$, is $\nu$-strongly convex, $L$-Lipschitz, and $M$-Hessian Lipschitz, meaning: $\forall \bbeta_1,
\bbeta_2\in\RR^p$,
\begin{itemize}
\item $\nu$ strongly convex: 
\[
f(\bbeta_2;  \bm{x}, y) \geq f(\bbeta_1; \bx, y) + \nabla f(\bbeta_1; \bm{x}, y ) (\bm{\beta}_2 - \bm{\beta}_1) + \frac{\nu}{2} \|\bm{\beta}_1 - \bm{\beta}_2\|_2^2.   
\]
\item Lipschitzness: 
\[
|f( {\bbeta}_2; \bx, y) - f({\bbeta}_1;  \bm{x}, y )| \leq L \|\bbeta_1- \bbeta_2\|_2. 
\]
\item M-Hessian Lipschitzness:
\[
\|\nabla^2 f(\bbeta_1;  \bm{x}, y)-\nabla^2 f(\bbeta_2; \bm{x}, y)\|_2 \leq M  \|\bbeta_1- \bbeta_2\|_2.
\]
\end{itemize}

Also, \cite{sekhari2021remember} considered a slightly relaxed version of certifiability that they call $(\epsilon, \delta)$-certifiability which is defined in the following way. Again as before, we present their definition in our notations: \\

\noindent\textbf{Definition 2 of \cite{sekhari2021remember}}
For all delete requests $\mathcal{M}$ of size at most $m$, and any $\mathcal{T} \subset \Theta$, the learning algorithm $A$ and unlearning algorithm $\tA$ satisfy $(\epsilon, \delta)$-unlearning property if and only if:
\[
\mathbb{P} \left( \tA(\cD_{\cM},A(\cD), T(\cD),\bb)  \in \cT \right) \leq {\rm e}^{\epsilon}  \mathbb{P}\left( \tA(\emptyset,A(\cDm), T(\cDm),\bb) \in \cT\right)+ \delta,
\]
and 
\[
 \mathbb{P}\left( \tA(\emptyset,A(\cDm), T(\cDm),\bb) \in \cT\right)  \leq {\rm e}^{\epsilon} \mathbb{P} \left( \tA(\cD_{\cM},A(\cD), T(\cD),\bb)  \in \cT \right) + \delta.
\]

Based on the assumptions above, \cite{sekhari2021remember} has proved the following theorem. \begin{theorem}\cite{sekhari2021remember}
\label{thm: sekhari_l2_error} 
Consider a machine unlearning estimate obtained by adding i.i.d. Gaussian noise (with variance specified in Algorithm 1 of the paper) to the estimate \( \tbm^{(1)} \) produced by a single step of the Newton method. Under Assumption 1 of \cite{sekhari2021remember} mentioned above, and assuming that the elements of the dataset are i.i.d., we have 
\begin{enumerate}
\item  Approximation accuracy of a single Newton method:
\[
    \Vert\tbm^{(1)} - \hbm \Vert_2\leq \frac{2ML^2m^2}{\nu^3n^2}.
\]
\item The unlearning algorithm satisfies $(\epsilon, \delta)$-unlearning. 

\item For any subset $\mathcal{M}$ of size less than or equal to $m$:
\[
\mathbb{E} |F(\tbm^{R,1}) - \min_{\bbeta} F(\bbeta)| = O \left(\frac{\sqrt{p} M m^2 L^3}{\nu^3 n^2 \epsilon} \sqrt{\log \left(\frac{1}{\epsilon}\right)} + \frac{4mL^2}{\nu n}\right). 
\]
\end{enumerate}
\end{theorem}

If one assumes that $L= O(1)$, $M= O(1)$, and $\nu = O(1)$, as is implicitly assumed in most of the conclusions in \cite{sekhari2021remember} it seems that a single step of the Newton method is sufficient for ensuring that:
\[
\mathbb{E} |F(\tbm^{R,1}) - \min_{\bbeta} F(\bbeta)| \rightarrow 0,
\]
as $p, n \rightarrow \infty$, as long as $m = o (\frac{n}{p^{1/4}})$. This conclusion is in fact mentioned in the abstract of \cite{sekhari2021remember}. Our claim is that
\begin{itemize}
    \item One cannot assume that $L$, $M$ and $\nu$ are $O(1)$ in high-dimensional settings. These quantities are, in fact, expected to depend on \(n, p \). It is therefore important to account for such dependencies in the theoretical analysis.
   
    \item Once we obtain the correct order of these three parameters, we will notice that the bounds of \cite{sekhari2021remember} are not sharp for high-dimensional settings. 
    
\end{itemize}

  In the remainder of this section, we aim to incorporate these considerations into a refined analysis. Below we provide a detailed description of Assumption 1 of \cite{sekhari2021remember} mentioned above. To make our discussion clear, similar to \cite{guo2019certified} we focus on the ridge regularizer,
  \[
  r(\bbeta) = \|\bbeta\|_2^2. 
  \]

\begin{enumerate}

\item Strong convexity assumption: The authors of \cite{sekhari2021remember} assume $f$ to be $\nu$-strongly convex in $\bbeta$, which means the empirical loss function $\hat{F}_n(\bbeta)$ is $\nu$-strongly convex.
Note that 
\[
    \nabla^2 f(\bbeta; \bx, y) = \ldd(y|\bx^\top\bbeta)\bx\bx^\top + \frac{\lambda}{n}  I. 
\]
Furthermore, $\ldd(y|\bx^\top\bbeta)\bx\bx^\top$ is a rank-one matrix. Hence, it is straightforward to see that $\nu = O(\frac{1}{n})$. 


\item Lipschitzness of the loss function: Below we perform some heuristic calculations to suggest what the order of the Lipschitz constant can be as a function of $n,p$. By using the mean value theorem
\begin{align*}
    |\ell(y|\bx^\top\bbeta_1) - \ell(y|\bx^\top\bbeta_2)|
    &= |\ld(y| \bx^\top\bxi)\bx^\top(\bbeta_1-\bbeta_2)|\\
    &\leq |\ld(y| \bx^\top\bxi)| \Vert\bx\Vert \Vert\bbeta_1-\bbeta_2\Vert,
\end{align*}
where $\bxi$ is a point on the line that connects $\bbeta_1$ and $\bbeta_2$. 
Hence, in order to understand the order of the Lipschitz constant we should understand the orders of $\Vert\bx_i\Vert$ and $|\ld(y_i| \bx_i^\top\bxi)|$. Since we want the Lipschitz property to hold for every $(\bx, y)$, using Cauchy-Schwartz inequality doesn't compromise the sharpness as equality can be attained for $\bbeta_1 = \bbeta_2 + t\bx$ for some $t\in\RR$.

Note that according to Assumption \ref{assum:normality} we conclude that $\|\bx\|_i = O_p(1)$. Moreover, it can be shown that $|\ld (y_i | \bx_i^{\top} \bm{\bxi})| = O_p(1)$ for a wide range of $\bxi$ that we concern (details can be found in Appendix \ref{lem: P(F_234)}), then assuming that the Lipschitz constant does not grow is an acceptable assumption. So far, we have ignored the regularizer. Note that
\[
\left\vert\frac{\lambda}{n}\|\bbeta_1\|^2 - \frac{\lambda}{n} \|\bbeta_2\|^2\right\vert 
\leq \frac{\lambda}{n} \|\bbeta_1+ \bbeta_2\| \|\bbeta_1 -\bbeta_2\|. 
\]
If we assume that each elements of $\bbeta$ is bounded by a constant, then  $\frac{\lambda}{n} \|\bbeta_1+ \bbeta_2\| = O(\frac{\sqrt{p}}{n})$. Hense even under PHAS, we can assume that the Lipschitz constant of $f(\bbeta ; y, \bx)$ remains $O(1)$ for all values of $(y,\bx,\bbeta)$ of interest. 

\item Hessian-Lipschitzness of the loss function: By straightforward calculations we obtain that 
\[
    \nabla^2 f(\bbeta; y, \bx) = \ldd(y|\bx^\top\bbeta)\bx\bx^\top + \frac{\lambda}{n} I.
\]
Then, we have
\begin{align*}
    \Vert\nabla^2  f(\bbeta_1; y, \bx) - \nabla^2 f(\bbeta_2; y, \bx) \Vert 
    & = \Vert\ldd(y|\bx^\top\bbeta_1)\bx\bx^\top - \ldd(y|\bx^\top\bbeta_2)\bx\bx^\top  \Vert\\
    &\simeq |\lddd(y | \xi)| \Vert\bx\Vert)\Vert\bbeta_1-\bbeta_2\Vert,
\end{align*}
so it is acceptable to assume $f$ is $O(1)$-Hessian-Lipschitz.


\end{enumerate}

Putting all the scalings above together, we have:
\begin{align*}
    \nu =&O\left(\frac{1}{n}\right)\,;\,
    M = O(1)\,;\,    L = O(1).  
\end{align*}
Therefore the bound in Theorem \ref{thm: sekhari_l2_error} becomes 
\[
    \Vert\tbm^{(1)} - \hbm \Vert_2\leq \frac{2ML^2m^2n}{\nu^3} = O(m^2n).
\]
In contrast, our Theorem \ref{thm: one-step-new} shows that 
\[
    \Vert\tbm^{(1)} - \hbm \Vert_2 = O\left(\sqrt{\frac{m^3}{n}}\right),
\]
which goes to zero even when $m$ growns with $n$ slowly enough.

Similarly, the excess risk of Theorem \ref{thm: sekhari_l2_error}  becomes:
\[
\mathbb{E} (F(\tbm^{R,1}) - \min_{\bbeta} F(\bbeta)) = O \left(\frac{\sqrt{p} m^2 n}{\epsilon} \sqrt{\log \left(\frac{1}{\epsilon}\right)} \right). 
\]
Note that the bounds on the excess risk is proportional to $n^{3/2}$ even when we set $m=1$, hence it does not provide useful information about the accuracy of the approximations. 
We should mention that in this paper we did not work with the excess risk, and instead we worked with the GED (Definition \ref{def: gen error div}), since the excess risk diverges in high dimensional R-ERM even for exact removal $\hbm$ without perturbation. To see this, consider a simple model of linear regression with $\ell_2$ loss, and $y_i |\bx_i\sim\cN(\bx_i^\top\bbeta^*,\sigma^2)$, $\bx_i\sim \cN(0,\frac1n\II_p)$. Then 
\begin{align*}
    f(y|\bx^\top\bbeta) &=  (y-\bx^\top\bbeta)^2 \\
    F(\bbeta) &= \EE f(y|\bx^\top\bbeta)=\sigma^2 + \frac1n \|\bbeta-\bbeta^*\|^2\\
    \min_{\bbeta} F(\bbeta) &= \sigma^2\\
    \EE|F(\hbm) - \min_{\bbeta} F(\bbeta)| &= \frac1n \|\hbm - \bbeta^*\|^2.
\end{align*}
It is known that in high dimensions, for most R-ERM estimators including the MLE, $\frac1n \|\hbm - \bbeta^*\|^2 \to \alpha_*>0$ (e.g. \cite{miolane2021distribution,thrampoulidis2018precise}), so in general the excess risk does not converge to zero  under PHAS.
In contrast, the Generalization Error Divergence (GED) we defined still converges to 0 under PHAS by Theorem \ref{thm: main_accuracy}.

\subsection{Literature on approximate leave-one-out cross validation}

Another line of work related to our paper focuses on efficient approximations of the leave-one-out cross-validation (LO) estimate of the risk. LO is widely known to provide an accurate estimate of out-of-sample prediction error (\cite{rad2020LO}). However, a major limitation of LO is its high computational cost. As a result, recent studies have explored methods for approximating LO more efficiently \citep{beirami2017optimal, stephenson2020approximate, rad2018scalable, giordano2019swiss, giordano2019higher, wang2018approximate, rad2020LO, patil2021uniform, patil2022}.

Among this body of work, the contributions of \cite{rad2018scalable, rad2020LO} are most closely related to our own. Specifically, \cite{rad2018scalable} demonstrated that approximating the LO solution using a single Newton step yields an estimate that falls within the statistical error of the true LO estimate.

While there are some technical parallels, our work differs from theirs in several key ways:  
(1) The criteria considered in this paper differ fundamentally from those in \cite{rad2018scalable}, where the focus is solely on risk estimation and not on privacy-related concerns. Consequently, the notion of certifiability, which is central to our analysis, was not considered in that line of work.  
(2) We consider a more general setting involving the removal of \( m \) data points, rather than a single leave-one-out sample.  
(3) As we show in Theorem~\ref{thm: main_accuracy}, due to the stricter requirements of our certifiability criterion, a single Newton step is no longer sufficient, and multiple steps are necessary to achieve a reliable approximation.

\section{Numerical Experiments}
We present numerical experiments that validate our theoretical findings regarding the perturbation  scale $r_{t,n}$, as well as the accuracy of the one-step and two-step Newton approximations, as functions of \( n \), \( p \), and \( m \). Specifically, we empirically examine the scaling behavior of \( \| \hat{\bm{\beta}} - \hat{\bm{\beta}}_{/\mathcal{M}} \|_2 \), \( \| \tilde{\bm{\beta}}^{(1)}_{/\mathcal{M}} - \hat{\bm{\beta}}_{/\mathcal{M}} \|_2 \), and \( \| \tilde{\bm{\beta}}^{(2)}_{/\mathcal{M}} - \hat{\bm{\beta}}_{/\mathcal{M}} \|_2 \) with respect to these parameters. Furthermore, we show that injecting noise into the one-step and two-step Newton approximations can provide privacy guarantees with and without significantly compromising the informativeness of the estimates, respectively. Finally, we demonstrate how our method can be applied to real-world datasets and practical problems, illustrating its empirical performance. The code for reproducing our experimental results is available at \href{https://github.com/krad-zz/Certified-Machine-Unlearning}{https://github.com/krad-zz/Certified-Machine-Unlearning}.


\begin{figure}[h]
\includegraphics[scale=0.27]{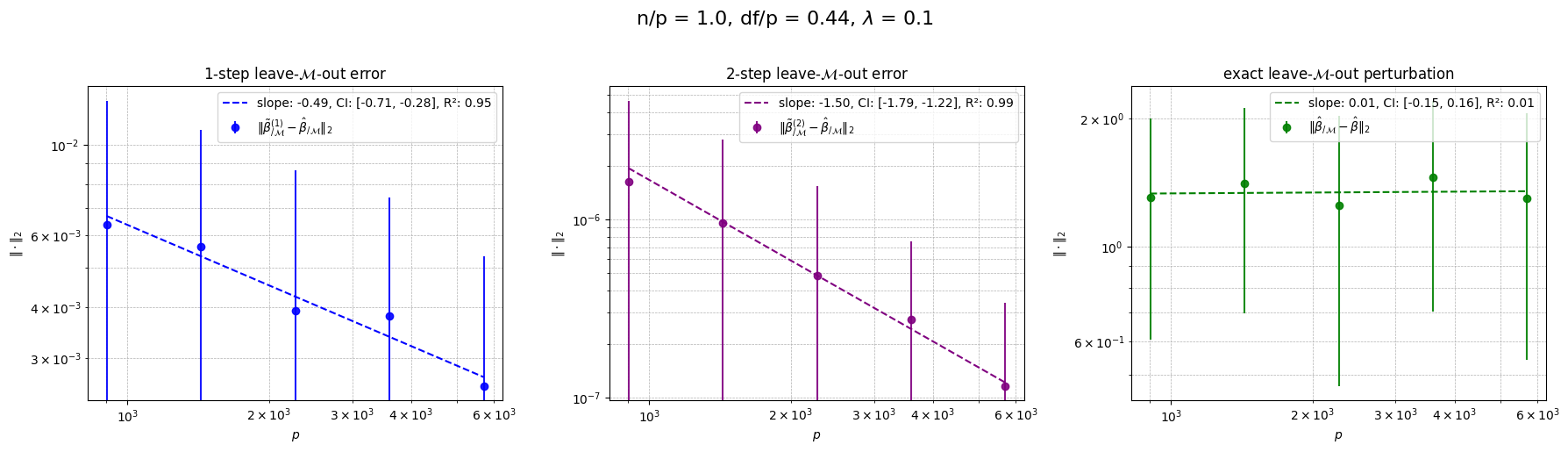}
\caption{The approximation and exact removal error for ridge logistic regression as function of $p$. To compare the one and two Newton step performances, no noise is added, leading to non-certified machine unlearning. Left: The one Newton step approximation error $\| \tbm^{(1)} - \hbm \|_2$. Middle: The two Newton step approximation error $\|\tbm^{(2)} - \hbm \|_2$. Right: The exact removal error $\| \hbm - \hb \|_2$.} 
\label{fig:p_scaling_npequal}
\end{figure}

\begin{figure}[h]
\includegraphics[scale=0.27]{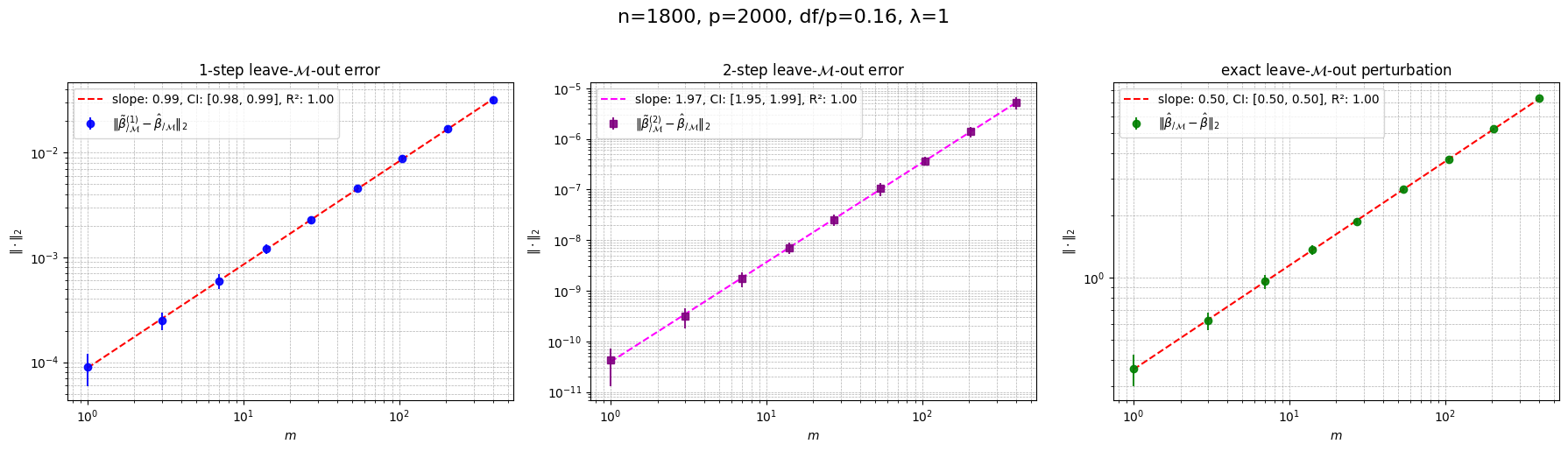}
\caption{The approximation and exact removal error error for ridge logistic regression as function of $m=|\mathcal{M}|$. Left: The one Newton step approximation error $\| \tbm^{(1)} - \hbm \|_2$. Middle: The two Newton step approximation error $\|\tbm^{(2)} - \hbm \|_2$. Right: The exact removal error $\| \hbm - \hb \|_2$.} \label{fig:m_scaling}
\end{figure}


\begin{figure}[h]
\begin{center}
\includegraphics[scale=0.3]{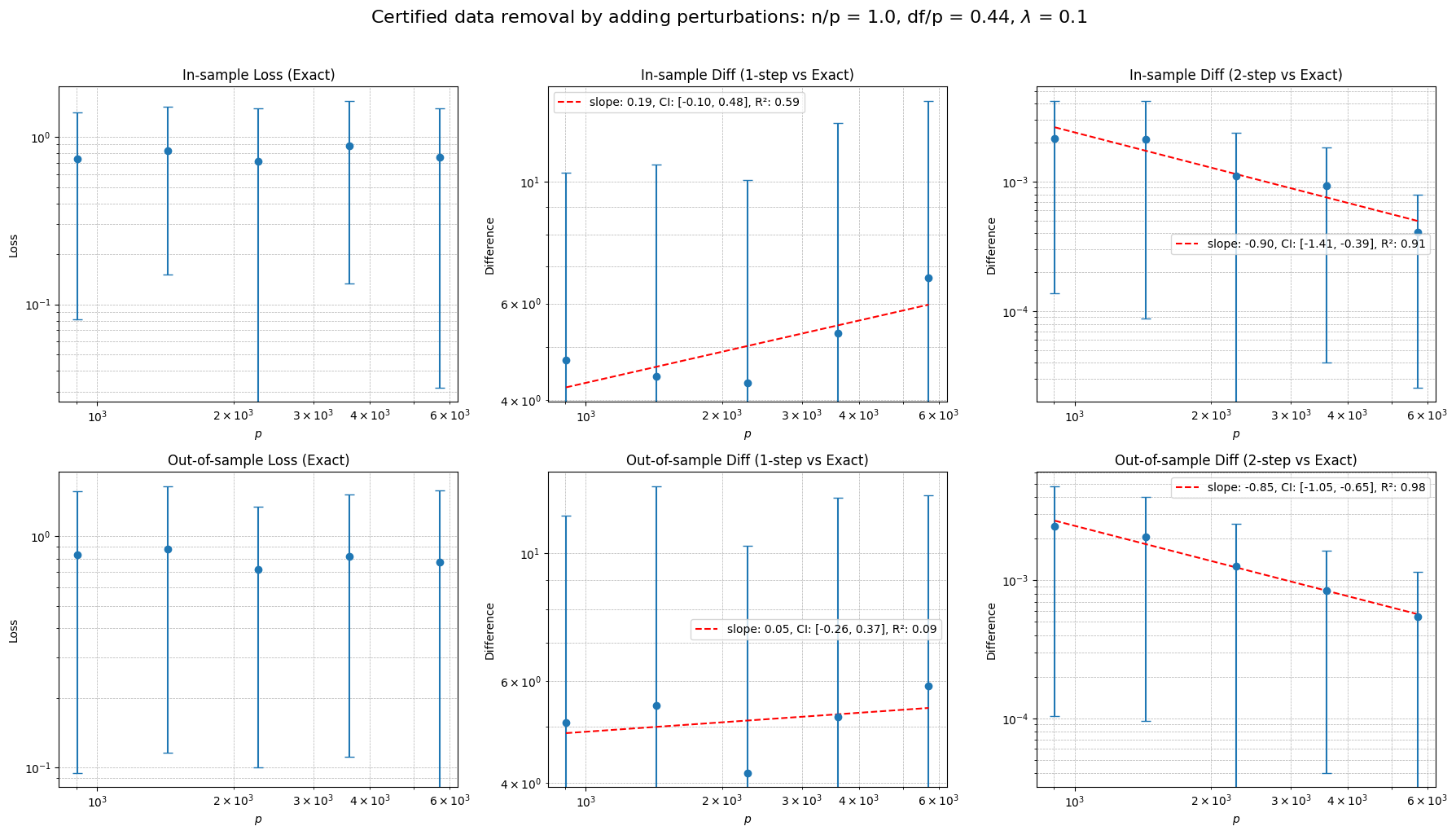}
\end{center}
\caption{Comparisons between the loss of the exactly unlearned models and its first and second Newton step approximations for ridge logistic regression as function of $p$. Top row: Tested on the left out data point. Left: loss of the exactly unlearned model. Middle: absolute error between the exactly unlearned model and its one Newton approximation. Right: absolute error between the exactly unlearned model and its second Newton approximation. Buttom row: Tested on an unseen data point. Left: loss of the exactly unlearned model. Middle: absolute error between the exactly unlearned model and its one Newton approximation. Right: absolute error between the exactly unlearned model and its second Newton approximation.} \label{fig:loss_withnoise_npequal}
\end{figure}

We let the true unknown parameter vector \( \bm{\beta}^* \sim \mathcal{N}(\bm{0}, \II_p) \), and we generate the feature vectors as \( \bm{x} \sim \mathcal{N}(\bm{0}, \II_n/n) \), which implies that \( \operatorname{var}(\bm{x}^\top \bm{\beta}^*) = p/n \). This is consistent with the finite signal to noise, high dimensional setting considered in this paper. We sample the responses as $$ y_i \sim \operatorname{Bernoulli}\left(\frac{1}{1 + e^{-\bm{x}_i^\top \bm{\beta}^*}} \right).$$ We used 100 MCMC samples to compute means and standard deviations.


\subsection{Non-certified Machine Unlearning}

While the no-noise setting does not constitute certified machine unlearning, we include it to empirically observe the scaling behavior of the one-step and two-step Newton approximations presented in Theorem \ref{thm: one-step-new} and \ref{thm: l2_error_multistep}. Figure \ref{fig:p_scaling_npequal} illustrates the following scalings (for fixed $m$):
\begin{eqnarray*}
    \| \tbm^{(1)} - \hbm \|_2 &=& O \left ( p^{-0.5}\right) \\
    \| \tbm^{(2)} - \hbm \|_2 &=& O \left ( p^{-1.5} \right).
\end{eqnarray*}
Figure \ref{fig:m_scaling} illustrates the following scalings (for fixed $p$):
\begin{eqnarray*}
    \| \tbm^{(1)} - \hbm \|_2 &=& O \left (m \right) \\
    \| \tbm^{(2)} - \hbm \|_2 &=& O \left ( m^2 \right) \\
    \| \hb - \hbm \|_2 &=& O \left (\sqrt{m} \right).
\end{eqnarray*}

These observations are consistent with the scalings presented in Theorem \ref{thm: one-step-new}.



\subsection{Certified Machine Unlearning}
This case requires adding noise to ensure the data to be unlearned is effectively masked. Noise is added via a random vector $\bm{b}_T \sim \mathrm{Gamma}(p, \epsilon / r_{T,n})$, where $\epsilon = 0.1$ and
\[
r_{T,n} = \max_{|\mathcal{M}| \leq m} \left\| \hbm - \tbm^{(T)} \right\|_2.
\]
Computing $r_{T,n}$  requires evaluating all ${n \choose |\mathcal{M}|}$ possible subsets, which is computationally infeasible even for moderate $n$. 

For the case $|\mathcal{M}| = 1$, instead of exhaustively enumerating all ${n \choose |\mathcal{M}|}$ configurations, we select a random subset of size $m_0$ and compute the maximum $\|\hbm - \tbm^{(T)}\|_2$ over this subset. To approximate the global maximum, we rescale the result by a factor of $\sqrt{\log{n \choose |\mathcal{M}|} / \log m_0}$.

Figure \ref{fig:loss_withnoise_npequal} shows that the amount of noise required for certified unlearning in a single step Newton model is so large that it not only removes the targeted information, but also degrades parts of the model that should have been retained. In contrast, the two-step Newton model requires significantly less noise, allowing it to effectively remove only the targeted information while retaining the rest of the model’s learned patterns. The in-sample error is defined as
\[
\frac{1}{|\mathcal{M}|} \sum_{i \in \mathcal{M}} \left| \ell\left(y_i \mid \bx_i^\top (\bm{b}_T + \tbm^{(T)})\right) - \ell\left(y_i \mid \bx_i^\top ( \hbm)\right) \right|,
\]
and the out-of-sample error is given by
\[
\left| \ell\left(y_0 \mid \bx_0^\top (\bm{b}_T + \tbm^{(T)})\right) - \ell\left(y_0 \mid \bx_0^\top ( \hbm)\right) \right|
\]
where $\bx_o \sim N(\bm{0}, \bm{I}/n)$ and $y_0 \sim {\rm Binomial}\left(\frac{1}{1+e^{-\bx_0^\top \bbeta^*}}\right)$ are new unseen data points.

\bibliographystyle{plainnat}
\bibliography{main}

\begin{thebibliography}{64}
\providecommand{\natexlab}[1]{#1}
\providecommand{\url}[1]{\texttt{#1}}
\expandafter\ifx\csname urlstyle\endcsname\relax
  \providecommand{\doi}[1]{doi: #1}\else
  \providecommand{\doi}{doi: \begingroup \urlstyle{rm}\Url}\fi

\bibitem[Adler and Taylor(2007)]{adler2009}
Robert~J. Adler and Jonathan~E. Taylor.
\newblock \emph{Random Fields and Geometry}.
\newblock Springer New York, NY, 2007.

\bibitem[Amelunxen et~al.(2013)Amelunxen, Lotz, McCoy, and Tropp]{AmLoMcTr13}
D.~Amelunxen, M.~Lotz, M.~B. McCoy, and J.~A. Tropp.
\newblock Living on the edge: A geometric theory of phase transitions in convex
  optimization.
\newblock \emph{arXiv preprint arXiv:1303.6672}, 2013.

\bibitem[Auddy et~al.(2024)Auddy, Zou, Rahnama~Rad, and Maleki]{auddy24a}
Arnab Auddy, Haolin Zou, Kamiar Rahnama~Rad, and Arian Maleki.
\newblock Approximate leave-one-out cross validation for regression with
  $\ell_1$ regularizers.
\newblock \emph{IEEE Transactions on Information Theory}, 70\penalty0
  (11):\penalty0 8040--8071, 2024.

\bibitem[Bayati and Montanari(2011)]{BaMo10}
M.~Bayati and A.~Montanari.
\newblock The dynamics of message passing on dense graphs, with applications to
  compressed sensing.
\newblock \emph{IEEE Trans. Inform. Theory}, 57\penalty0 (2):\penalty0
  764--785, 2011.

\bibitem[Bayati and Montanari(2012)]{BaMo11}
M.~Bayati and A.~Montanari.
\newblock The {LASSO} risk for {G}aussian matrices.
\newblock \emph{IEEE Trans. Inform. Theory}, 58\penalty0 (4):\penalty0
  1997--2017, 2012.

\bibitem[Beirami et~al.(2017)Beirami, Razaviyayn, Shahrampour, and
  Tarokh]{beirami2017optimal}
Ahmad Beirami, Meisam Razaviyayn, Shahin Shahrampour, and Vahid Tarokh.
\newblock On optimal generalizability in parametric learning.
\newblock \emph{Advances in neural information processing systems}, 30, 2017.

\bibitem[Bourtoule et~al.(2021)Bourtoule, Chandrasekaran, Choquette-Choo, Jia,
  Travers, Zhang, Lie, and Papernot]{bourtoule2021machine}
Lucas Bourtoule, Varun Chandrasekaran, Christopher~A Choquette-Choo, Hengrui
  Jia, Adelin Travers, Baiwu Zhang, David Lie, and Nicolas Papernot.
\newblock Machine unlearning.
\newblock In \emph{2021 IEEE Symposium on Security and Privacy (SP)}, pages
  141--159. IEEE, 2021.

\bibitem[Boyd and Vandenberghe(2004)]{BoydBook}
S.~P. Boyd and L.~Vandenberghe.
\newblock \emph{Convex optimization}.
\newblock Cambridge university press, 2004.

\bibitem[Cao and Yang(2015)]{cao2015towards}
Yinzhi Cao and Junfeng Yang.
\newblock Towards making systems forget with machine unlearning.
\newblock In \emph{2015 IEEE symposium on security and privacy}, pages
  463--480. IEEE, 2015.

\bibitem[Celentano and Montanari(2024)]{celentano2024correlation}
Michael Celentano and Andrea Montanari.
\newblock Correlation adjusted debiased lasso: debiasing the lasso with
  inaccurate covariate model.
\newblock \emph{Journal of the Royal Statistical Society Series B: Statistical
  Methodology}, 86\penalty0 (5):\penalty0 1455--1482, 2024.

\bibitem[Chatterjee(2014)]{chatterjee2014}
Sourav Chatterjee.
\newblock A new perspective on least squares under convex constraint.
\newblock \emph{The Annals of Statistics}, 42\penalty0 (6):\penalty0
  2340--2381, 2014.

\bibitem[Chen et~al.(2021)Chen, Zhang, Wang, Backes, Humbert, and
  Zhang]{chen2021machine}
Min Chen, Zhikun Zhang, Tianhao Wang, Michael Backes, Mathias Humbert, and Yang
  Zhang.
\newblock When machine unlearning jeopardizes privacy.
\newblock In \emph{Proceedings of the 2021 ACM SIGSAC conference on computer
  and communications security}, pages 896--911, 2021.

\bibitem[Chundawat et~al.(2023)Chundawat, Tarun, Mandal, and
  Kankanhalli]{chundawat2023zero}
Vikram~S Chundawat, Ayush~K Tarun, Murari Mandal, and Mohan Kankanhalli.
\newblock Zero-shot machine unlearning.
\newblock \emph{IEEE Transactions on Information Forensics and Security},
  18:\penalty0 2345--2354, 2023.

\bibitem[Dobriban and Liu(2019)]{dobriban2019asymptotics}
Edgar Dobriban and Sifan Liu.
\newblock Asymptotics for sketching in least squares regression.
\newblock \emph{Advances in Neural Information Processing Systems}, 32, 2019.

\bibitem[Dobriban and Wager(2018)]{dobriban2018high}
Edgar Dobriban and Stefan Wager.
\newblock High-dimensional asymptotics of prediction: Ridge regression and
  classification.
\newblock \emph{The Annals of Statistics}, 46\penalty0 (1):\penalty0 247--279,
  2018.

\bibitem[Donoho et~al.(2009)Donoho, Maleki, and Montanari]{DoMaMo09}
D.~L. Donoho, A.~Maleki, and A.~Montanari.
\newblock Message passing algorithms for compressed sensing.
\newblock \emph{Proc. Natl. Acad. Sci.}, 106\penalty0 (45):\penalty0
  18914--18919, Sep. 2009.

\bibitem[Donoho et~al.(2011)Donoho, Maleki, and Montanari]{DoMaMoNSPT}
D.~L. Donoho, A.~Maleki, and A.~Montanari.
\newblock The noise-sensitivity phase transition in compressed sensing.
\newblock \emph{IEEE Trans. Inform. Theory}, 57\penalty0 (10):\penalty0
  6920--6941, 2011.

\bibitem[Donoho and Montanari(2016)]{donoho2016high}
David Donoho and Andrea Montanari.
\newblock High dimensional robust m-estimation: Asymptotic variance via
  approximate message passing.
\newblock \emph{Probability Theory and Related Fields}, 166\penalty0
  (3-4):\penalty0 935--969, 2016.

\bibitem[Dudeja et~al.(2023)Dudeja, M.~Lu, and Sen]{dudeja2023universality}
Rishabh Dudeja, Yue M.~Lu, and Subhabrata Sen.
\newblock Universality of approximate message passing with semirandom matrices.
\newblock \emph{The Annals of Probability}, 51\penalty0 (5):\penalty0
  1616--1683, 2023.

\bibitem[Dwork(2006)]{dwork2006differential}
Cynthia Dwork.
\newblock Differential privacy.
\newblock In \emph{International colloquium on automata, languages, and
  programming}, pages 1--12. Springer, 2006.

\bibitem[El~Karoui et~al.(2013)El~Karoui, Bean, Bickel, Lim, and
  Yu]{el2013robust}
Noureddine El~Karoui, Derek Bean, Peter~J Bickel, Chinghway Lim, and Bin Yu.
\newblock On robust regression with high-dimensional predictors.
\newblock \emph{Proceedings of the National Academy of Sciences}, 110\penalty0
  (36):\penalty0 14557--14562, 2013.

\bibitem[Fan(2022)]{fan2022approximate}
Zhou Fan.
\newblock Approximate message passing algorithms for rotationally invariant
  matrices.
\newblock \emph{The Annals of Statistics}, 50\penalty0 (1):\penalty0 197--224,
  2022.

\bibitem[Foster et~al.(2024)Foster, Schoepf, and Brintrup]{foster2024fast}
Jack Foster, Stefan Schoepf, and Alexandra Brintrup.
\newblock Fast machine unlearning without retraining through selective synaptic
  dampening.
\newblock In \emph{Proceedings of the AAAI conference on artificial
  intelligence}, volume~38, pages 12043--12051, 2024.

\bibitem[Fourdrinier et~al.(2018)Fourdrinier, Strawderman, and
  Wells]{fourdrinier2018shrinkage}
Dominique Fourdrinier, William~E Strawderman, and Martin~T Wells.
\newblock \emph{Shrinkage estimation}.
\newblock Springer, 2018.

\bibitem[Giordano et~al.(2019{\natexlab{a}})Giordano, Jordan, and
  Broderick]{giordano2019higher}
Ryan Giordano, Michael~I Jordan, and Tamara Broderick.
\newblock A higher-order swiss army infinitesimal jackknife.
\newblock \emph{arXiv preprint arXiv:1907.12116}, 2019{\natexlab{a}}.

\bibitem[Giordano et~al.(2019{\natexlab{b}})Giordano, Stephenson, Liu, Jordan,
  and Broderick]{giordano2019swiss}
Ryan Giordano, William Stephenson, Runjing Liu, Michael Jordan, and Tamara
  Broderick.
\newblock A swiss army infinitesimal jackknife.
\newblock In \emph{The 22nd International Conference on Artificial Intelligence
  and Statistics}, pages 1139--1147. PMLR, 2019{\natexlab{b}}.

\bibitem[Goel et~al.(2022)Goel, Prabhu, Sanyal, Lim, Torr, and
  Kumaraguru]{goel2022towards}
Shashwat Goel, Ameya Prabhu, Amartya Sanyal, Ser-Nam Lim, Philip Torr, and
  Ponnurangam Kumaraguru.
\newblock Towards adversarial evaluations for inexact machine unlearning.
\newblock \emph{arXiv preprint arXiv:2201.06640}, 2022.

\bibitem[Graves et~al.(2021)Graves, Nagisetty, and Ganesh]{graves2021amnesiac}
Laura Graves, Vineel Nagisetty, and Vijay Ganesh.
\newblock Amnesiac machine learning.
\newblock In \emph{Proceedings of the AAAI Conference on Artificial
  Intelligence}, volume~35, pages 11516--11524, 2021.

\bibitem[Guo et~al.(2019)Guo, Goldstein, Hannun, and Van
  Der~Maaten]{guo2019certified}
Chuan Guo, Tom Goldstein, Awni Hannun, and Laurens Van Der~Maaten.
\newblock Certified data removal from machine learning models.
\newblock \emph{arXiv preprint arXiv:1911.03030}, 2019.

\bibitem[Gupta et~al.(2021)Gupta, Jung, Neel, Roth, Sharifi-Malvajerdi, and
  Waites]{gupta2021adaptive}
Varun Gupta, Christopher Jung, Seth Neel, Aaron Roth, Saeed Sharifi-Malvajerdi,
  and Chris Waites.
\newblock Adaptive machine unlearning.
\newblock \emph{Advances in Neural Information Processing Systems},
  34:\penalty0 16319--16330, 2021.

\bibitem[Izzo et~al.(2021)Izzo, Anne~Smart, Chaudhuri, and Zou]{izzo21a}
Zachary Izzo, Mary Anne~Smart, Kamalika Chaudhuri, and James Zou.
\newblock Approximate data deletion from machine learning models.
\newblock In Arindam Banerjee and Kenji Fukumizu, editors, \emph{Proceedings of
  The 24th International Conference on Artificial Intelligence and Statistics},
  volume 130 of \emph{Proceedings of Machine Learning Research}, pages
  2008--2016. PMLR, 13--15 Apr 2021.
\newblock URL \url{https://proceedings.mlr.press/v130/izzo21a.html}.

\bibitem[Jalali and Maleki(2016)]{jalali2016}
Shirin Jalali and Arian Maleki.
\newblock New approach to bayesian high-dimensional linear regression.
\newblock \emph{Information and Inference: A Journal of the IMA}, 7, 07 2016.

\bibitem[Karoui and Purdom(2016)]{karoui2016can}
Noureddine~El Karoui and Elizabeth Purdom.
\newblock Can we trust the bootstrap in high-dimension?
\newblock \emph{arXiv preprint arXiv:1608.00696}, 2016.

\bibitem[Krzakala et~al.(2012{\natexlab{a}})Krzakala, M{\'e}zard, Sausset, Sun,
  and Zdeborov{\'a}]{KrMeSaSuZd12}
F.~Krzakala, M.~M{\'e}zard, F.~Sausset, Y.~Sun, and L.~Zdeborov{\'a}.
\newblock Statistical-physics-based reconstruction in compressed sensing.
\newblock \emph{Physical Review X}, 2\penalty0 (2):\penalty0 021005,
  2012{\natexlab{a}}.

\bibitem[Krzakala et~al.(2012{\natexlab{b}})Krzakala, M{\'e}zard, Sausset, Sun,
  and Zdeborov{\'a}]{KrMeSaSuZd12b}
F.~Krzakala, M.~M{\'e}zard, F.~Sausset, Y.~Sun, and L.~Zdeborov{\'a}.
\newblock Probabilistic reconstruction in compressed sensing: algorithms, phase
  diagrams, and threshold achieving matrices.
\newblock \emph{J. Stat. Mechanics: Theory and Experiment}, 2012\penalty0
  (08):\penalty0 P08009, 2012{\natexlab{b}}.

\bibitem[Kurmanji et~al.(2023)Kurmanji, Triantafillou, Hayes, and
  Triantafillou]{kurmanji2023towards}
Meghdad Kurmanji, Peter Triantafillou, Jamie Hayes, and Eleni Triantafillou.
\newblock Towards unbounded machine unlearning.
\newblock \emph{Advances in neural information processing systems},
  36:\penalty0 1957--1987, 2023.

\bibitem[Laurent and Massart(2000)]{laurent2000adaptive}
Beatrice Laurent and Pascal Massart.
\newblock Adaptive estimation of a quadratic functional by model selection.
\newblock \emph{Annals of statistics}, pages 1302--1338, 2000.

\bibitem[Li et~al.(2025)Li, Zhou, Gao, Chen, Zhang, Kuang, and Fu]{li25survey}
Na~Li, Chunyi Zhou, Yansong Gao, Hui Chen, Zhi Zhang, Boyu Kuang, and Anmin Fu.
\newblock Machine unlearning: Taxonomy, metrics, applications, challenges, and
  prospects.
\newblock \emph{IEEE Transactions on Neural Networks and Learning Systems},
  pages 1--21, 2025.

\bibitem[Li and Wei(2021)]{li2021minimum}
Yue Li and Yuting Wei.
\newblock Minimum $\ell_1$-norm interpolators: Precise asymptotics and multiple
  descent.
\newblock \emph{arXiv preprint arXiv:2110.09502}, 2021.

\bibitem[Liang and Sur(2022)]{liang2022precise}
Tengyuan Liang and Pragya Sur.
\newblock A precise high-dimensional asymptotic theory for boosting and minimum
  $\ell_1$-norm interpolated classifiers.
\newblock \emph{The Annals of Statistics}, 50\penalty0 (3):\penalty0
  1669--1695, 2022.

\bibitem[Maleki(2011)]{MalekiThesis}
A.~Maleki.
\newblock Approximate message passing algorithm for compressed sensing.
\newblock \emph{Stanford University Ph.D. Thesis}, 2011.

\bibitem[Maleki and Montanari(2010)]{MaMoCISS10}
A.~Maleki and A.~Montanari.
\newblock Analysis of approximate message passing algorithm.
\newblock In \emph{Proc. IEEE Conf. Inform. Science and Systems (CISS)}, 2010.

\bibitem[Miolane and Montanari(2021)]{miolane2021distribution}
L{\'e}o Miolane and Andrea Montanari.
\newblock The distribution of the lasso: Uniform control over sparse balls and
  adaptive parameter tuning.
\newblock \emph{The Annals of Statistics}, 49\penalty0 (4), 2021.

\bibitem[Mousavi et~al.(2017)Mousavi, Maleki, and
  Baraniuk]{mousavi2013asymptotic}
Ali Mousavi, Arian Maleki, and Richard~G Baraniuk.
\newblock Consistent parameter estimation for {LASSO} and approximate message
  passing.
\newblock \emph{Annals of Statistics}, 45\penalty0 (6):\penalty0 2427–2454,
  2017.

\bibitem[Neel et~al.(2021)Neel, Roth, and Sharifi-Malvajerdi]{neel2021descent}
Seth Neel, Aaron Roth, and Saeed Sharifi-Malvajerdi.
\newblock Descent-to-delete: Gradient-based methods for machine unlearning.
\newblock In \emph{Algorithmic Learning Theory}, pages 931--962. PMLR, 2021.

\bibitem[Nguyen et~al.(2022)Nguyen, Huynh, Ren, Nguyen, Liew, Yin, and
  Nguyen]{nguyen2022survey}
Thanh~Tam Nguyen, Thanh~Trung Huynh, Zhao Ren, Phi~Le Nguyen, Alan Wee-Chung
  Liew, Hongzhi Yin, and Quoc Viet~Hung Nguyen.
\newblock A survey of machine unlearning.
\newblock \emph{arXiv preprint arXiv:2209.02299}, 2022.

\bibitem[Oymak et~al.(2013)Oymak, Thrampoulidis, and Hassibi]{oymak2013squared}
Samet Oymak, Christos Thrampoulidis, and Babak Hassibi.
\newblock The squared-error of generalized lasso: A precise analysis.
\newblock In \emph{Proc. Annual Allerton Conference on Communication, Control,
  and Computing}, pages 1002--1009. IEEE, 2013.

\bibitem[Patil et~al.(2021)Patil, Wei, Rinaldo, and
  Tibshirani]{patil2021uniform}
Pratik Patil, Yuting Wei, Alessandro Rinaldo, and Ryan Tibshirani.
\newblock Uniform consistency of cross-validation estimators for
  high-dimensional ridge regression.
\newblock In \emph{International Conference on Artificial Intelligence and
  Statistics}, pages 3178--3186. PMLR, 2021.

\bibitem[Patil et~al.(2022)Patil, Rinaldo, and Tibshirani]{patil2022}
Pratik Patil, Alessandro Rinaldo, and Ryan Tibshirani.
\newblock Estimating {F}unctionals of the {O}ut-of-{S}ample {E}rror
  {D}istribution in {H}igh-{D}imensional {R}idge {R}egression.
\newblock In \emph{Proceedings of The 25th International Conference on
  Artificial Intelligence and Statistics}, volume 151 of \emph{Proceedings of
  Machine Learning Research}, pages 6087--6120. PMLR, 2022.

\bibitem[Pawelczyk et~al.(2024)Pawelczyk, Di, Lu, Sekhari, Kamath, and
  Neel]{pawelczyk2024machine}
Martin Pawelczyk, Jimmy~Z Di, Yiwei Lu, Ayush Sekhari, Gautam Kamath, and Seth
  Neel.
\newblock Machine unlearning fails to remove data poisoning attacks.
\newblock \emph{arXiv preprint arXiv:2406.17216}, 2024.

\bibitem[Rahnama~Rad and Maleki(2020)]{rad2018scalable}
Kamiar Rahnama~Rad and Arian Maleki.
\newblock A scalable estimate of the out-of-sample prediction error via
  approximate leave-one-out cross-validation.
\newblock \emph{Journal of the Royal Statistical Society Series B: Statistical
  Methodology}, 82\penalty0 (4):\penalty0 965--996, 2020.

\bibitem[Rahnama~Rad et~al.(2020)Rahnama~Rad, Zhou, and Maleki]{rad2020LO}
Kamiar Rahnama~Rad, Wenda Zhou, and Arian Maleki.
\newblock Error bounds in estimating the out-of-sample prediction error using
  leave-one-out cross validation in high-dimensions.
\newblock In Silvia Chiappa and Roberto Calandra, editors, \emph{Proceedings of
  the Twenty Third International Conference on Artificial Intelligence and
  Statistics}, volume 108 of \emph{Proceedings of Machine Learning Research},
  pages 4067--4077. PMLR, 26--28 Aug 2020.

\bibitem[Sekhari et~al.(2021)Sekhari, Acharya, Kamath, and
  Suresh]{sekhari2021remember}
Ayush Sekhari, Jayadev Acharya, Gautam Kamath, and Ananda~Theertha Suresh.
\newblock Remember what you want to forget: Algorithms for machine unlearning.
\newblock \emph{Advances in Neural Information Processing Systems},
  34:\penalty0 18075--18086, 2021.

\bibitem[Stephenson and Broderick(2020)]{stephenson2020approximate}
William Stephenson and Tamara Broderick.
\newblock Approximate cross-validation in high dimensions with guarantees.
\newblock In \emph{International Conference on Artificial Intelligence and
  Statistics}, pages 2424--2434. PMLR, 2020.

\bibitem[Suriyakumar and Wilson(2022)]{suriyakumar2022algorithms}
Vinith Suriyakumar and Ashia~C Wilson.
\newblock Algorithms that approximate data removal: New results and
  limitations.
\newblock \emph{Advances in Neural Information Processing Systems},
  35:\penalty0 18892--18903, 2022.

\bibitem[Tarun et~al.(2023)Tarun, Chundawat, Mandal, and
  Kankanhalli]{tarun2023fast}
Ayush~K Tarun, Vikram~S Chundawat, Murari Mandal, and Mohan Kankanhalli.
\newblock Fast yet effective machine unlearning.
\newblock \emph{IEEE Transactions on Neural Networks and Learning Systems},
  2023.

\bibitem[Thrampoulidis et~al.(2018)Thrampoulidis, Abbasi, and
  Hassibi]{thrampoulidis2018precise}
Christos Thrampoulidis, Ehsan Abbasi, and Babak Hassibi.
\newblock Precise error analysis of regularized $ m $-estimators in high
  dimensions.
\newblock \emph{IEEE Transactions on Information Theory}, 64\penalty0
  (8):\penalty0 5592--5628, 2018.

\bibitem[Wainwright(2019)]{wainwright19HDS}
Martin~J. Wainwright.
\newblock \emph{High-Dimensional Statistics: A Non-Asymptotic Viewpoint}.
\newblock Cambridge Series in Statistical and Probabilistic Mathematics.
  Cambridge University Press, 2019.

\bibitem[Wang et~al.(2018)Wang, Zhou, Lu, Maleki, and
  Mirrokni]{wang2018approximate}
Shuaiwen Wang, Wenda Zhou, Haihao Lu, Arian Maleki, and Vahab Mirrokni.
\newblock Approximate leave-one-out for fast parameter tuning in high
  dimensions.
\newblock In \emph{International Conference on Machine Learning}, pages
  5228--5237. PMLR, 2018.

\bibitem[Wang et~al.(2020)Wang, Weng, and Maleki]{WangWengMaleki2020}
Shuaiwen Wang, Haolei Weng, and Arian Maleki.
\newblock {Which bridge estimator is the best for variable selection?}
\newblock \emph{The Annals of Statistics}, 48\penalty0 (5):\penalty0 2791 --
  2823, 2020.

\bibitem[Wang et~al.(2022)Wang, Weng, and Maleki]{wang2022does}
Shuaiwen Wang, Haolei Weng, and Arian Maleki.
\newblock Does slope outperform bridge regression?
\newblock \emph{Information and Inference: A Journal of the IMA}, 11\penalty0
  (1):\penalty0 1--54, 2022.

\bibitem[Weng et~al.(2018)Weng, Maleki, and Zheng]{WengMalekiZheng18}
Haolei Weng, Arian Maleki, and Le~Zheng.
\newblock {Overcoming the limitations of phase transition by higher order
  analysis of regularization techniques}.
\newblock \emph{The Annals of Statistics}, 46\penalty0 (6A):\penalty0 3099 --
  3129, 2018.

\bibitem[Xu et~al.(2023)Xu, Zhu, Zhang, Zhou, and Yu]{xu24survey}
Heng Xu, Tianqing Zhu, Lefeng Zhang, Wanlei Zhou, and Philip~S. Yu.
\newblock Machine unlearning: A survey.
\newblock \emph{ACM Comput. Surv.}, 56\penalty0 (1), August 2023.
\newblock ISSN 0360-0300.
\newblock URL \url{https://doi.org/10.1145/3603620}.

\bibitem[Zou et~al.(2024)Zou, Auddy, Rahnama~Rad, and
  Maleki]{zou2024theoretical}
Haolin Zou, Arnab Auddy, Kamiar~\ Rahnama~Rad, and Arian Maleki.
\newblock Theoretical analysis of leave-one-out cross validation for
  non-differentiable penalties under high-dimensional settings.
\newblock \emph{arXiv preprint arXiv:2402.08543}, 2024.

\end{thebibliography}

\appendix

\section{Detailed Proofs}

Before we dive into the proof details, we first provide a proof sketch so that our proof strategy could appear more clear to the reader, which also helps the reader navigate through the rest of the appendix.

Recall that by our definition, $\tbm^{R,t} = \tbm^{(t)}+\bb$, where $\tbm^{(t)}$ is the t-step Newton approximation for $\hbm$, and $p(\bb)\propto e^{-\frac{\epsilon}{r_{t,n}}\Vert\bb\Vert} $. The ultimate goal is to prove that 
\begin{enumerate}
    \item (Theorem \ref{thm: main_epscert}) $\tbm^{R,t}$ achieves $(\phi_n,\epsilon)$-PAR for some $\phi_n\to 0 $ if 
    \[
        r_{n,t} = O\left(\left(\frac{m^3}{n}\right)^{2^{t-2}}\polylog(n)\right).
    \]
    \item (Theorem \ref{thm: main_accuracy}) ${\rm GED}^{\epsilon}(\tbm^{R,t},\hbm) = O(\frac{\sqrt{mp}}{\eps}r_{t,n}\polylog(n))$.
\end{enumerate}

According to Section \ref{sssec: trade-off}, the key to both is the $\ell_2$ error of the $t$-step Newton estimator, i.e., $\Vert\tbm^{(t)} - \hbm \Vert_2$. To bound this quantity for $t\geq 1$, we first provided a bound for $t=1$ case in Theorem \ref{thm: one-step-new}:
\[
    \Vert\tbm^{(1)} - \hbm\Vert_2 = o_p\left(\sqrt{\frac{m^3}{n}}\polylog(n)\right),
\]
and when $m=o(n^{1/3})$, the general case can be then obtained by the quadratic convergence property of Newton method (Lemma \ref{lem:t step newton general}):
\[
    \Vert\tbm^{(t)} - \hbm \Vert_2 \simeq C \Vert\tbm^{(t-1)}-\hbm \Vert_2^2,\quad t\geq 2,
\]
where $C$ depends on the Lipschitzness of the Hessian and the strong convexity of the objective function.
So we get the general bound $\Vert \tbm^{(t)}-\hbm\Vert_2=o_p\left(\left(\frac{m^3}{n}\right)^{2^{t-2}}\polylog(n)\right)$. This is exactly the $r_{n,t}$ in Theorem $\ref{thm: main_epscert}$ and Theorem \ref{thm: main_accuracy}.

\subsection{Proof of Lemma \ref{lem: direct_perturbation}}
\label{ssec: proof_lem_direct_perturbation}
\begin{proof}
    We first prove sufficiency. Conditional on the data $\cD$, $\tbm+\bb$ is nothing but a translation of $\bb$, so its conditional density is
    \[
        p_{\tbm+\bb|\cD}(\bm{\beta}) = p_{\bb}(\bm{\beta} - \tbm),
    \]
    and it is similar for $\hbm+\bb$. 
    Notice that $\log(p_{\bb}(\bb))=-\frac{\epsilon}{r}\Vert\bb\Vert$ is $\frac{\epsilon}{r}$-Lipschitz in $\bb$, therefore $\forall~ \cD\in\mathcal{X}_r$,
    \begin{align*}
        &~|\log(p_{\tbm+\bb}(\bm{\beta})) - \log(p_{\hbm+\bb}(\bm{\beta}))|\\
        =&~ |\log(p_{\bb}(\bm{\beta} - \tbm)) - \log(p_{\bb}(\bm{\beta} - \hbm))|\\
        \leq &~ \frac{\epsilon}{r}\Vert\tbm-\hbm\Vert\\
        \leq&~\epsilon
    \end{align*}
    which is equivalent to
    \[
        {\rm e}^{-\epsilon}\leq \frac{p_{\tbm+\bb}(\bm{\beta})}{p_{\hbm+\bb}(\bm{\beta})}\leq {\rm e}^{\epsilon}.
    \]
    The result then follows by integrating the densities over the set $\cT$.

    For the necessity part, for $\cD_0\notin \cX_r$, by definition $\exists \cM\subset [n], \exists \delta>0$, $\Vert\hbm-\tbm\Vert_2\geq r+\delta$. The proof is then straightforward since $-\frac{\epsilon}{r}\Vert\bb\Vert$ is not $\frac{\epsilon}{r+\delta}$-Lipschitz. To be more specific, define
    \begin{align*}
        \mathcal{T}_+&:=\{\bbeta\in \RR^p | \Vert\bbeta - \hbm\Vert - \Vert\bbeta-\tbm\Vert\geq r+\frac{\delta}{2}\},\\
        \mathcal{T}_-&:=\{\bbeta\in \RR^p | \Vert\bbeta - \tbm\Vert - \Vert\bbeta-\hbm\Vert\geq r+\frac{\delta}{2}\}.
    \end{align*}
    By basic geometry this is the areas within two pieces of a hyperboloid. Then $\forall T\subset \mathcal{T}_+$, 
    \[
        \frac{p_{\tbm+\bb}(\bbeta)}{p_{\hbm+\bb}(\bbeta)} = {\rm e}^{\frac{\epsilon}{r}(\Vert\bbeta-\hbm \Vert - \Vert\bbeta-\tbm\Vert)}
        \geq {\rm e}^{\epsilon (1+\frac{\delta}{2r})},
    \]
    and similarly $\forall T\subset \mathcal{T}_-$, 
    \[
        \frac{p_{\tbm+\bb}(\bbeta)}{p_{\hbm+\bb}(\bbeta)} \leq {\rm e}^{-\epsilon (1+\frac{\delta}{2r})}.
    \]
\end{proof}

\medskip

\subsection{$\ell_2$ error of t-step Newton estimators}
\label{ssec: l2-newton}
In this section, we prove Theorem \ref{thm: one-step-new} and Theorem \ref{thm: l2_error_multistep}. In fact we prove some statements slightly stonger than we discussed in the beginning of the Appendix: not only can we bound the $\ell_2$ errors of t-step Newton estimators, but we can bound them \textbf{simutaneously}. To be more specific, we define a ``failure event'' $F$ for the dataset $\cD$, such that:
\begin{enumerate}
    \item on event $F$, 
    \[
        \|\tbm^{(1)} - \hbm\|_2 = O\left(\sqrt{\frac{m^3}{n}}\polylog(n)\right)
    \]
    (Lemma \ref{lem:l2-norm-diff}),
    \item $\PP(F)=\phi_n \to 0$ (Lemma \ref{lem: P(F)}), and 

    \item on event $F$, the Hessian $\bGm$ is $\polylog(n)$-Lipschitz in $\bbeta$, so that the Newton estimators satisfy quadratic convergence (Lemma \ref{lem:t step newton general}) and we get
    \[  
        \Vert \tbm^{(t)}-\hbm\Vert_2=O\left(\left(\frac{m^3}{n}\right)^{2^{t-2}}\polylog(n)\right),   \quad \forall ~t\geq 1
    \]
    (Theorem \ref{thm: l2_error_multistep}).
\end{enumerate}
In the rest of the section, we adopt the following notations for brevity:
\begin{align*}
    \bldd(\bbeta) &= [\ldd_i(\bbeta)]_{i\in[n]}, \Ldd(\bbeta) = \diag[\bldd(\bbeta)]\\
    \bldd_{\backslash\cM}(\bbeta) &= [\ldd_i(\bbeta)]_{i\notin\cM},\Ldd_{\backslash\cM} = \diag[\bldd_{\backslash\cM}(\bbeta)]\\
    \brdd(\bbeta) &= [\rdd_k(\beta_k)]_{k\in[p]},
    \Rdd(\bbeta) =\diag[\brdd(\bbeta)]\\
    \bG(\bbeta) &= \bX^\top\Ldd(\bbeta)\bX+\lambda\Rdd(\bbeta)\\
    \bGm(\bbeta)&=\bXm^\top\Ldd(\bbeta)\bXm+\lambda\Rdd(\bbeta).
\end{align*}
\subsubsection{$\ell_2$ error of one Newton step}
\label{sssec: one-step}
We first provide a finer characterization of $\cX_r^{(1)}=\{\Vert\tbm^{(1)}-\hbm \Vert\leq r\}$ in terms of some smaller ``failure events'' $F_i, i=1,2,...,5$ that are easier to verify:
\begin{align*}
    F_1&:=\{\Vert\bX \Vert>C_1\},\\
    F_2&:=\{\max_{i\in[n]}|\ld_i(\hb)|>C_\ell(n)\},\\
    F_3&:= \{\exists \cM\subset\cD \text{ with } |\cM|\leq m, \bbeta\in\cB_{1,\cM}, \\
    &\qquad\Vert\bldd_{\backslash\cM}(\bbeta) - \bldd_{\backslash\cM}(\hbm) \Vert>C_{\ell\ell}(n)\Vert\bbeta-\hbm \Vert\\
    &\qquad\text{ or } \Vert\nabla^2 r(\bbeta) - \nabla^2 r(\hbm) \Vert>C_{rr}(n)\Vert\bbeta-\hbm \Vert \}
    \end{align*}    
\begin{align*}
    \text{wh}&\text{ere } \bxi(t):=t\hb+(1-t)\hbm,\; \bldd_{\backslash\cM}(\bbeta):=[\ldd_i(\bbeta)]_{i\notin\cM},\\   &\quad \cB_{1,\cM}:=\{\bbeta: \bbeta = t\hb+(1-t)\hbm, t\in[0,1]\}. \\
    F_4&:= \{ \exists \cM\subset\cD \text{ with } |\cM|\leq m, \bbeta_1,\bbeta_2\in\cB_{2,\cM},\\
    &\qquad\Vert\bldd_{\backslash\cM}(\bbeta_1) - \bldd_{\backslash\cM}(\bbeta_2) \Vert>C_{\ell\ell}(n)\Vert\bbeta_1-\bbeta_2 \Vert,\\
    &\qquad\text{ or } \Vert\nabla^2 r(\bbeta) - \nabla^2 r(\hbm) \Vert>C_{rr}(n)\Vert\bbeta-\hbm \Vert \}\\
    \text{wh}&\text{ere } \cB_{2,\cM}:=\cB(\hbm ,1)\\
    F_5&:= \{\max_{|\cM|\leq m}\Vert\bbXm \bGm^{-1}(\hbm)\bX_{\cM}^\top \Vert_{2,\infty}>C_{xx}(n)\sqrt{\frac{m}{n}}\}, 
    \\ 
    \text{wh}&\text{ere }
    \bbXm = \begin{pmatrix}
        \bXm\\\II_p
    \end{pmatrix}
    ,\; \bX_{\cM}^\top = (\bx_i)_{i\in\cM}\\
    F&:=\cup_{i=1}^5 F_i
    \label{eq: def_failure_events}\numberthis
\end{align*}
\begin{remark}
    Technically, $F_4$ will not be used in this section, but will be used in the proof of Theorem~\ref{thm: l2_error_multistep} in Section \ref{sssec: multi-step}. However, we enlist it among other events as it is similar to $F_3$ and the proof is similar too.
\end{remark}

Note that the undetermined constants, namely $C_1$, $C_\ell(n)$, $C_{\ell\ell}(n)$, $C_{rr}(n)$ and $C_{xx}(n)$, will be decided later when we analyze their probabilities. Roughly speaking, these constants will be at most $O(\polylog(n))$ under Assumptions B1-B3. The next lemma shows that $\left(\cX_r^{(1)}\right)^c\subset F:=\cup_{i=1}^5 F_i$ for a specific $r$:
\begin{lemma}\label{lem:l2-norm-diff}
    under Assumptions A1-A3, Let $F=\cup_{i=1}^5 F_i$ be the failure event, then $\left(\cX_r^{(1)}\right)^c\subset F$ with 
    \[
        r= \left[\frac{2\sqrt{3}}{3\lambda^2\nu^2}[C_{\ell\ell}(n)+\lambda C_{rr}(n)]C_1(C_1+1)C_\ell^2(n)C_{xx}(n) \right]\frac{m^{3/2}}{\sqrt{n}}:=C_1(n)\frac{m^{3/2}}{\sqrt{n}}
    \]    
    i.e., under $F^c$, 
    \[
        \max_{|\cM|\leq m}\Vert\tbm^{(1)} - \hbm \Vert\leq C_1(n)\frac{m^{3/2}}{\sqrt{n}}.
    \]
\end{lemma}
The proof of this Lemma is postponed to Section \ref{sssec: proof_l2-norm-diff}.

Now that we have $\left(\cX_r^{(1)}\right)^c\subset F=\cup_{i=1}^5 F_i $, we can bound $\phi=\PP(\cD\notin\cX_r^{(1)})$ by bounding $\PP(F)$ instead:

\begin{lemma}\label{lem: P(F)}
    Under Assumptions A1-A3 and B1-B3, 
    \[
        \PP(F)\leq nq_n^{(y)} + 8n^{1-c} + ne^{-p/2} + 2e^{-p}
    \]
    where the constants $C_\ell(n), C_{\ell\ell}(n), C_{rr}(n)$ and $C_{xx}(n)$ in \eqref{eq: def_failure_events} are all $O(\polylog(n))$.
\end{lemma}
The proof is in Section \ref{ssssec: proof_lem_P(F)} and the exact form of the constants can be found in the proof of Lemma \ref{lem: P(F_234)} and \ref{lem: P(F_5)}. Theorem \ref{thm: one-step-new} is then a direct application of Lemma \ref{lem:l2-norm-diff} and Lemma \ref{lem: P(F)}:
\begin{proof}[Proof of Theorem \ref{thm: one-step-new}]
    By Lemma \ref{lem:l2-norm-diff}, $(\cX_r^{(1)})^c\subset F$ with $r = C_1(n) \frac{m^{\frac32}}{\sqrt{n}}$. By Lemma \ref{lem: P(F)}, $\PP(F)\leq nq_n^{(y)} + 8n^{1-c} + ne^{-p/2} + 2e^{-p}$. Combining the two lemmas we immediately have
    \[
        \PP(\cD\in\cX_r^{(1)})\leq nq_n^{(y)} + 8n^{1-c} + ne^{-p/2} + 2e^{-p},
    \]
    where $r = C_1(n) \frac{m^{\frac32}}{\sqrt{n}}$ and  $C_1(n)=O(\polylog(n))$ by Lemma \ref{lem: P(F)},
    which concludes the proof of Theorem \ref{thm: one-step-new}.
\end{proof}

\subsubsection{Proof of Lemma \ref{lem:l2-norm-diff}}
\label{sssec: proof_l2-norm-diff}

\begin{proof}
    Recall that we denote $\bG(\bbeta)$ and $\bGm(\bbeta)$ to be the Hessian of the loss functions of the full model and the unlearned model respectively, and \begin{align*}
        \bbG&:= \int_0^1 \bG(t\hb + (1-t)\hbm)dt\\
        \bbGm&:=  \int_0^1 \bGm(t\hb + (1-t)\hbm)dt
    \end{align*}
    
    By Lemma \ref{lem:beta_lo_error}, we have, using the notations above and in \eqref{eq: def_failure_events},
    \[
        \hbm - \hb = \bbGm^{-1} \left(\sum_{i\in\cM}\ld_i(\hb)\bx_i\right)=\bbGm^{-1}\bX_{\cM}^\top\bld_{\cM} 
    \]
    and by Definition \ref{def: Newton} we have
    \[
        \tbm^{(1)} - \hb = \bGm^{-1}(\hb)\left(\sum_{i\in\cM}\ld_i(\hb)\bx_i\right) = \bGm^{-1}(\hb)\bX_{\cM}^\top\bld_{\cM}.
    \]
    If we define $\bvm:=\bX_{\cM}^\top\bld_{\cM}$, then by subtracting the two equations above, we have
    \begin{align*}
        \hbm - \tbm^{(1)} &= \left[\bbGm^{-1} - \bGm^{-1}(\hb)\right]\bvm\\
        &= \left[\bbGm^{-1} - \bGm^{-1}(\hbm)\right]\bvm + \left[\bGm^{-1}(\hbm)-\bGm^{-1}(\hb)\right]\bvm  \\
        &:= \bM_1\bvm + \bM_2\bvm,
    \end{align*}
    thus we have
    \[
        \Vert\hbm - \tbm^{(1)} \Vert\leq \Vert\bM_1\bvm \Vert + \Vert\bM_1\bvm\Vert.
    \]
    Since $\bM_1$ and $\bM_2$ possess similar properties, we will bound $\Vert\bM_1\bvm\Vert$ in the following, while $\Vert\bM_2\bvm\Vert$ can be bounded using the same method.
    \begin{align*}
        \bM_1\bvm &=
        \left[\bbGm^{-1} - \bGm^{-1}(\hbm)\right]\bvm 
        \\
        &=\bbGm^{-1}\left[\bGm(\hbm) - \bbGm\right]\bGm^{-1}(\hbm)\bvm
    \end{align*}
    Notice that we can write $\bGm(\hbm) - \bbGm$ in a more compact form:
    \begin{align*}
        \bGm(\hbm) - \bbGm &= \bXm^\top [\Ldd(\hbm) - \bLdd]\bXm + \lambda [\Rdd(\hbm)-\bRdd]\\
        &:=\bbXm^\top \bGamma \bbXm,
    \end{align*}
    where in the last step we define 
    \begin{align*}
        \bbXm&:=
        \begin{pmatrix}
            \bXm\\
            \bI_p
        \end{pmatrix},\\
        \bGamma &:=
        \begin{pmatrix}
            \diag[\int_0^1 \ldd_i(\bxi(t)) - \ldd_i(\hbm) dt]_{i\notin\cM},& \bzero\\
            \bzero,&\diag[\int_0^1 \rdd(\xi_k(t)) - \rdd(\hat{\beta}_{\lambda\backslash\cM,k})]_{k\in[p]}
        \end{pmatrix},\\
        \bxi(t)&:=t\hb + (1-t)\hbm
        \label{eq:def_bGamma}\numberthis.
    \end{align*}
    We then have
    \begin{align*}
        &\Vert\bM_1\bvm\Vert\\
        =& \sup_{\Vert\bw\Vert=1} |\bw^\top\bM_1\bvm|\\
        =& \sup_{\Vert\bw\Vert=1} \left|\bw^\top \bbGm^{-1} \bbXm^\top\bGamma \bbXm \bGm^{-1}(\hbm)\bvm  \right|\\
        \leq& \sup_{\Vert\bw\Vert=1} \Vert\bGamma\Vert_{Fr}\Vert\bbXm\bbGm^{-1}\bw \Vert_2 \left\|\bbXm\bGm^{-1}(\hbm) \bv_{\cM} \right\|_\infty\\
        =& \Vert\bGamma\Vert_{Fr}\cdot \sup_{\Vert\bw\Vert=1}\Vert\bbXm\bbGm^{-1}\bw \Vert_2\cdot \left\|\bbXm\bGm^{-1}(\hbm)\bv_{\cM} \right\|_\infty,
        \numberthis\label{eq:proof_l2_error_{t,n}hree_terms}
    \end{align*}
    where in the penultimate line we use Cauchy Schwarz inequality: for two vectors $\bu,\bv\in\RR^p$ and diagonal matrix $\bD=\diag[d_k]_{k\in[p]}$,
    \begin{align*}
        \bu^\top\bD\bv = \sum_{k\in[p]} d_ku_kv_k
        \leq&~ \left(\sum_{k\in[p]}d_k^2\right)^{\frac12}\left(\sum_{k\in[p]}u_k^2v_k^2\right)^{\frac12}
        \le~ \left(\sum_{k\in[p]}d_k^2\right)^{\frac12}\left(\sum_{k\in[p]}u_k^2\right)^{\frac12}\max_{1\le k\le p}|v_k|\\
        =&~\|\bD\|_{\rm Fr}\|\bu\|_2\|\bv\|_{\infty}.
    \end{align*}
    Now we bound the three terms in \eqref{eq:proof_l2_error_{t,n}hree_terms} separately. 
    \begin{enumerate}
        \item 
        Define
        \begin{align*}
            \bldd_{\backslash\cM}(\bbeta)&:= [\ldd_i(\bbeta)]_{i\notin\cM}
            \\
            \bar{\bldd}_{\backslash\cM}&:= \left[\int_0^1\ldd_i(\bxi(t))dt\right]_{i\notin\cM} \text{ where } \bxi(t)=t\hb+(1-t)\hbm\\
            \brdd(\bbeta)&:= [\rdd(\beta_k)]_{k\in[p]}\\
            \bar{\brdd}&:=\left[\int_0^1 \brdd(\bxi(t))dt\right]_{k\in[p]},
        \end{align*}
        Then 
        \[
            \Vert\bGamma \Vert_{Fr}\leq \Vert\bar{\bldd}_{\backslash\cM} - \bldd_{\backslash\cM}(\hbm) \Vert_2 + \lambda \Vert \bar{\brdd} - \brdd(\hbm)\Vert_2.
        \]
        where
        \begin{align*}
            &~\Vert\bar{\bldd}_{\backslash\cM} - \bldd_{\backslash\cM}(\hbm) \Vert_2\\
            &= \sqrt{\sum_{i\notin\cM}\left[\int_0^1[\ldd_i(\bxi(t)-\ldd_i(\hbm))]\right]^2}
            \leq
            \sqrt{\int_0^1\sum_{i\notin\cM}[ \ldd_i(\bxi(t)-\ldd_i(\hbm)) ]^2dt}\\
            &\leq \sqrt{\int_0^1\Vert\bldd(\bxi(t))-\bldd(\hbm) \Vert^2 dt}
            \leq \sqrt{\int_0^1C_{\ell\ell}^2(n)\Vert\bxi(t)-\hbm \Vert^2dt}\\
            &=C_{\ell\ell}(n)\sqrt{\int_0^1 t^2 \Vert\hb-\hbm \Vert^2 dt }\\
            &= \frac{\sqrt{3}}{3}C_{\ell\ell}(n)\Vert\hb-\hbm \Vert.
        \end{align*}
        By Lemma \ref{lem:beta_lo_error}, under event $F$:
        \begin{align*}
            \Vert\hb-\hbm \Vert 
            &=\Vert\bbGm^{-1} \bX_{\cM}^\top\bld_{\cM}(\hb)\Vert
            \leq \frac{1}{\lambda\nu} \Vert\bX_{\cM}^\top\bld_{\cM}(\hb) \Vert\\
            &\leq \frac{1}{\lambda\nu} \Vert\bX \Vert\cdot \Vert\bld_{\cM}(\hb) \Vert
            \leq \frac{1}{\lambda\nu} \Vert\bX \Vert \sqrt{m} \max_{i\in[n]} |\ld_i(\hb)|\\
            &\leq \frac{\sqrt{m}}{\lambda\nu}C_1C_\ell(n),
        \end{align*}
        where the first line uses Lemma \ref{lem:beta_lo_error}, the second uses strong convexity of the risk function, and the last line uses the definition of events $F_1,F_2$. Therefore we have
        \[
            \Vert\bar{\bldd}_{\backslash\cM} - \bldd_{\backslash\cM}(\hbm) \Vert_2
            \leq
            \frac{\sqrt{3m}}{3\lambda\nu}C_1C_\ell(n)C_{\ell\ell}(n).
        \]
        Similarly we have 
        \[
            \Vert\bar{\brdd}_{\backslash\cM} - \brdd_{\backslash\cM}(\hbm) \Vert_2
            \leq
            \frac{\sqrt{3m}}{3\lambda\nu}C_1C_\ell(n)C_{rr}(n),
        \]
        so we have
        \[
            \Vert\bGamma \Vert_{Fr}\leq 
            \frac{\sqrt{3m}}{3\lambda\nu}[C_{\ell\ell}(n)+\lambda C_{rr}(n)]C_1C_\ell(n).
        \]
        \item 
        \begin{align*}
            \sup_{\Vert\bw\Vert=1}\Vert\bbXm\bbGm^{-1}\bw \Vert_2
            &\leq \sup_{\Vert\bw\Vert=1} \Vert\bbXm \Vert \Vert\bbGm^{-1} \Vert \Vert\bw \Vert\\
            &\leq \Vert\bbXm \Vert \Vert\bbGm^{-1} \Vert.
        \end{align*}
        Notice that $\Vert\bbXm \Vert\leq \Vert\bXm\Vert+1\leq \Vert\bX\Vert+1\leq C_1+1$, so we have
        \[
            \sup_{\Vert\bw\Vert=1}\Vert\bbXm\bbGm^{-1}\bw \Vert_2 \leq \frac{C_1+1}{\lambda\nu}
        \]
        \item 
        \begin{align*}
            &\left\|\bbXm\bGm^{-1}(\hbm)\bvm \right\|_\infty\\
            =&\left\| \bbXm\bGm^{-1}(\hbm) \bX_{\cM}^\top\bld_{\cM}(\hb) \right\|_\infty\\
            \leq & \left\|\bbXm\bGm^{-1}(\hbm) \bX_{\cM}^\top\right\|_{2,\infty} \left\|\bld_{\cM}(\hb)\right\|_2,
        \end{align*}
        where in the last line we use the definition of the $(2,\infty)$ norm of a matrix $\bA$:
        \[
            \Vert\bA \Vert_{2,\infty}:=\sup_{\Vert\bw \Vert_2\leq 1} \Vert\bA\bw \Vert_{\infty}.
        \]
        We know $\left\|\bld_{\cM}(\hb)\right\|_2\leq \sqrt{m}\max_{i\in\cM}|\ld_i(\hb)|\leq \sqrt{m}C_\ell(n)$, so by the definition of event $F_5$,
        \[
            \left\|\bbXm\bGm^{-1}(\hbm)\left(\sum_{i\in\cM}\ld_i(\hb)\bx_i\right) \right\|_\infty
            \leq
            C_\ell(n)C_{xx}(n)\frac{m}{\sqrt{n}}.
        \]
    \end{enumerate}
    Combining all the results above we have
    \[
        \Vert\bM_1\bvm \Vert\leq \frac{\sqrt{3}}{3\lambda^2\nu^2}[C_{\ell\ell}(n)+\lambda C_{rr}(n)]C_1(C_1+1)C_\ell^2(n)C_{xx}(n)\frac{m^{3/2}}{\sqrt{n}}.
    \]
    Similar arguments lead to the same bound for $\Vert\bM_2\bvm \Vert$. So we finally have, under event $F^c$,
    \begin{align*}
        \Vert\tbm^{(1)}-\hbm \Vert 
        \leq&~
        \frac{2\sqrt{3}}{3\lambda^2\nu^2}[C_{\ell\ell}(n)+\lambda C_{rr}(n)]C_1(C_1+1)C_\ell^2(n)C_{xx}(n)\frac{m^{3/2}}{\sqrt{n}}\\
        :=&~
        C_1(n)\frac{m^{3/2}}{\sqrt{n}}.
    \end{align*}
\end{proof}

\subsubsection{Proof of Lemma \ref{lem: P(F)}}
\label{ssssec: proof_lem_P(F)}
\begin{lemma}
\label{lem: P(F_234)}
    Under Assumptions A1-A3 and B1-B3, if we set 
    \begin{align*}
        C_1 &= (\sqrt{\gamma_0}+3)\sqrt{C_X},\\
        C_\ell(n)&=\polylog_3(n)\\
        C_{\ell\ell}(n)&=\max\{\polylog_{10}(n),\polylog_{11}(n)\},
    \end{align*}
    in the definition of $F_2, F_3$ in \eqref{eq: def_failure_events}, then we have
    \[
        \PP(F_1\cup F_2\cup F_3)\leq nq_n^{(y)} + 4n^{1-c} + ne^{-p/2} + e^{-p},
    \]   
\end{lemma}

The proof can be found in Section \ref{ssec: proof_lem_P(F_234)}, and the definition of the $\polylog(n)$ terms are summarized in \eqref{eq: def_all_polylogs}.

\begin{lemma}\label{lem: P(F_5)}
    Under Assumption \ref{assum:normality}, for any $c\ge 0$ we have $\PP(F_5)\leq 4n^{-c}+e^{-p}$ with 
    \[
        C_{xx}(n)=
    \frac{(\sqrt{\gamma_0}+3)C_X\vee 1}{\lambda\nu}
         \sqrt{\frac{m(1+4(c+1)\log(n))}{p}}
    \]
\end{lemma}

The proof of Lemma \ref{lem: P(F_5)} can be found in Section \ref{ssec: proof_lem_P(F_5)}. By combining the above two lemmata we have

\begin{proof}[Proof of Lemma \ref{lem: P(F)}]
    By Lemma \ref{lem: norm_X}, 
    \[\PP(\Vert\bX \Vert>(\sqrt{\gamma_0}+3)\sqrt{C_X})\leq e^{-p}. \]
    By Lemma \ref{lem: P(F_234)}, $\PP(F_2\cup F_3)\leq 2n^{1-c} + 2nq_n^{(y)} + e^{-p} + ne^{-p/2}$ with 
    \begin{align*}
        C_\ell(n)&=\polylog_3(n) = \polylog_1(n)+\frac{4C_X}{\lambda\nu}[1+C_y^s(n)+\polylog_1^s(n)]\\
        C_{\ell\ell}(n)&=\max\{\polylog_{10}(n),\polylog_{11}(n)\}.
    \end{align*}
    Finally by Lemma \ref{lem: P(F_5)}, $\PP(F_5)\leq 4n^{-c}+e^{-p}$ with 
    \[
        C_{xx}(n)=\frac{(\sqrt{\gamma_0}+3)C_X\vee 1}{\lambda\nu}
         \sqrt{\frac{m(1+4(c+1)\log(n))}{p}}.
    \]
    Using a union bound over the events above we have
    \[
        \PP(F)\leq nq_n^{(y)} + 8n^{1-c} + ne^{-p/2} + 2e^{-p}.
    \]
\end{proof}

\subsubsection{Proof of Lemma \ref{lem: P(F_234)}}
\label{ssec: proof_lem_P(F_234)}
\begin{proof}
    The proof can be divided into 4 steps. In the following we denote $\hbx{\cdot}$ to be the model trained by excluding the observations indicated in $\cdot$, which can be the indices like $\hbx{i,j}$, or a subset $\hbm$, or both, e.g. $\hbx{\cM,i}$. We allow the content in $\cdot$ to overlap, in which case we simply take a union, for example $\hbx{i,i}:=\hbx{i}$. By slight abuse of notation, we use the index ``0" as a placeholder to denote no observations being excluded, so that $\max_{0\leq i\leq n}\hbx{i}$ means the maximum among all $\hbi$ and also $\hb$. These additional definitions help us to keep the notations unified and simple.
    
    \noindent\underline{Step 1}

    Define event $E_1:=\{\max_i |y_i|\leq C_y(n)\}$ with probability at least $1-nq_n^{(y)}$ by Assumption B3 and a union bound. Under $E_1$, $\forall \cM\subset \cD$ (including $\cM=\emptyset$ case where $\hat{\bbeta}_{\lambda,\emptyset}:=\hb$). We have
    \begin{align*}
        \lambda\nu \Vert\hbm\Vert^2
        \overset{(a)}{\leq} &~
        \sum_{i\notin\cM}\ell_i(\hbm) + \lambda r(\hbm)\overset{(b)}{\leq} \sum_{i\in[n]}\ell_i(\bzero) \\
        \overset{(c)}{\leq} &~ \sum_{i\in[n]}1+|y_i|^s\leq n(1+C_y^s(n)),
    \end{align*}
    where $(a)$ uses the $\nu$-strong convexity of $r$ (Assumption A2), $(b)$ uses $r(\bzero)=0$ and $\ell(\cdot)\geq 0$, $(c)$ uses Asumption B2, and the last inequality uses the definition of event $E_1$. By rearranging the terms we have under $E_1$ that, $\forall \cM$
    \[
        \Vert\hbm \Vert\leq \sqrt{(\lambda\nu)^{-1}(1+C_y^s(n))n}.
        \label{eq: bd_all_betahat}\numberthis
    \]
    
    \noindent\underline{Step 2}
    
    Define event $E_2:=\{ \max_{i,j\in[n]} |\bx_i^\top \hbx{i,j}| \leq \polylog_1(n) \}$ with $\hbx{i,i}:=\hbi$ and 
    \[
    \polylog_1(n):=2\sqrt{(\lambda\nu)^{-1}\gamma_0C_X(1+C_y^s(n))(1+c)\log(n)}
    \]
    for any arbitrary $c>0$ that will appear in the tail probability.
    \begin{align*}
        &~\PP(E_1^c \cup E_2^c)\\
        &=\PP(E_1^c) + \PP(E_2^c\cap E_1)\\
        &\leq nq_n^{(y)} + \PP\left(\max_{i,j\in[n]} |\bx_i^\top \hbx{i,j}| > 2\sqrt{\frac{C_X}{p}}\Vert\hbx{i,j} \Vert\sqrt{(1+c)\log(n)}, \;E_1\right)\\
        &\leq nq_n^{(y)} +    \sum_{i\in[n]}\EE\PP\left( \max_{j\in[n]} |\bx_i^\top \hbx{i,j}| > 2\sqrt{\frac{C_X}{p}}\Vert\hbx{i,j} \Vert\sqrt{(1+c)\log(n)} | \cD_{\backslash i} \right)\\
        &\leq nq_n^{(y)} + 2n^{1-c},
    \end{align*}
    where the second line used $\Vert \hbx{i,j}\Vert\leq \sqrt{(\lambda\nu)^{-1}(1+C_y^s(n))n}$ under $E_1$, the third line uses a union bound over $i$ and the tower rule, and the last line uses Lemma \ref{lem: max_of_gaussian} by observing that conditional on $\cD_{\backslash i}$, $\max_{j\in[n]}|\bx_i^\top \hbx{i,j}|$ is the maximum of $n$ Gaussians with $\sigma_i^2 = \hbx{i,j}^\top\bSigma\hbx{i,j}\leq \frac{C_X}{p}\Vert\hbx{i,j} \Vert^2$. 

    \noindent\underline{Step 3}
    
    Define $E_3:=\{ \max_{i\in[n]} \Vert\bx_i\Vert\leq 2\sqrt{C_X}, \Vert\bX\Vert\leq (\sqrt{\gamma_0}+3)\sqrt{C_X} \}$ then $\PP(E_3^c)\leq ne^{-p/2} + e^{-p}$ by Lemma~\ref{lem: norm_X} and Lemma~\ref{lem: max_norm_x_i}. 
    
    Under $\cap_{i=1}^3E_i$, we have the following results: $\forall 1\leq i\leq n, 0\leq j\leq n$,
    \begin{align*}
        |\ld_i(\hbx{i,j})|
        &\leq 1+|y_i|^s + |\bx_i^\top \hbx{i,j}|^s\\
        &\leq 1+C_y^s(n) +\polylog_1^s(n):=\polylog_2(n)\\
        \Vert\hbx{j} - \hbx{i,j} \Vert&\leq \frac{1}{\lambda\nu}|\ld_i(\hbx{i,j})| \Vert\bx_i \Vert \leq \frac{2\sqrt{C_X}}{\lambda\nu}\polylog_2(n)\\
        |\bx_i^\top\hbx{j}|
        &\leq |\bx_i^\top \hbx{i,j}| + |\bx_i^\top(\hbx{j}-\hbx{i,j})|\\
        &\leq \polylog_1(n) + \frac{4C_X}{\lambda\nu}\polylog_2(n)\\
        &:=\polylog_3(n)\label{eq: proof_lem_P(F_234)_F_2}\numberthis\\
        |\ld_i(\hbx{j})|&\leq 1+C_y^s(n)+\polylog_3^s(n):=\polylog_4(n).
    \end{align*}

    \noindent\underline{Step 4}

    Under $\cap_{i=1}^3E_i$, $\forall \cM\subset\cD, |\cM|\leq m, \forall j\notin\cM$: by Lemma \ref{lem:beta_lo_error}, 
    \begin{align*}
        \Vert\hbx{\cM,j} - \hbx{j} \Vert&\leq \frac{1}{\lambda\nu}\Vert\bX_{\cM}\bld_{\cM}(\hbx{j}) \Vert\\
        &\leq \frac{\sqrt{m}}{\lambda\nu}\Vert\bX\Vert\max_{i\in\cM}|\ld_i(\hbx{j})|
        \leq \frac{\sqrt{m}}{\lambda\nu} (\sqrt{\gamma_0}+3)\sqrt{C_X}\polylog_4(n)
        \label{eq: proof_lem_P(F_234)_1}\numberthis.
    \end{align*}
    Define 
    \[
        E_4:=\{ \max_{j\in[n]}\max_{\underset{j\notin \cM}{|\cM|\leq m}}
            |\bx_i^\top(\hbx{\cM,j} - \hbx{j})|
            \leq
            \polylog_5(n)
        \}
    \]
    where $\polylog_5(n):= \frac{2(\sqrt{\gamma_0}+3)C_X}{\lambda\nu}\sqrt{m(2m+c)p^{-1}\log(n)}\polylog_4(n)$. Note that it is indeed $O(\polylog(n))$ when $m=o(\sqrt{p})$. Then 
    \[
        \PP(\cup_{i=1}^4 E_i^c)\leq \PP(\cup_{i=1}^3 E_i^c) + \PP(E_4^c \cap (\cap_{i=1}^3 E_i)),
    \]
    and now we work on the second term since the first is already known.
    Let $N:=\sum_{s=0}^{m-1}{n-1\choose s-1}\leq 2{n-1\choose m-1}=\frac{2m}{n}{n\choose m}$, then 
    \[
        \log(N)\leq \log(2) + \log(m)-\log(n)+m\log(em/n)\leq 2m\log(n)
    \]
    for $n\geq 8$, so that we have
    \[
        \sqrt{(2m+c)\log(n)}\geq \sqrt{\log(N)+c\log(n)}.
    \]
    Therefore we have
    \begin{align*}
        &\PP(E_4^c \cap (\cap_{i=1}^3 E_i))\\
        \overset{(a)}{\leq}& \PP(
            \max_{j\in[n]}\max_{\underset{j\notin \cM}{|\cM|\leq m}} 
            |\bx_i^\top(\hbx{\cM,j} - \hbx{j})|
            >
            2\sqrt{C_X/p} \Vert\hbx{\cM,j} - \hbx{j} \Vert\sqrt{\log(N)+c\log(n)}, \; \cap_{i=1}^3 E_i
            )\\
        \leq& \sum_{j\in[n]} \EE\PP(
            \max_{\underset{j\notin \cM}{|\cM|\leq m}}
            |\bx_i^\top(\hbx{\cM,j} - \hbx{j})|
            >
            2\sqrt{C_X/p} \Vert\hbx{\cM,j} - \hbx{j} \Vert\sqrt{\log(N)+c\log(n)}\; | \cD_{\backslash j}
            )\\
        \overset{(b)}{\leq}& 2n^{1-c},
    \end{align*}
    where in $(a)$ we used \eqref{eq: proof_lem_P(F_234)_1} under $\cap_{i=1}^3 E_i$, and in (b) we used Lemma \ref{lem: max_of_gaussian} again.

    We therefore have
    \[
        \PP(\cup_{i=1}^4E_i^c)\leq nq_n^{(y)} + 4n^{1-c} + ne^{-p/2} + e^{-p}.
    \]
    Under event $\cap_{i=1}^4 E_i$, $\forall i\in[n], \forall |\cM|\leq m, i\notin \cM$:
    \begin{align*}
        |\bx_i^\top\hbx{\cM,i}|&\leq |\bx_i^\top\hbi| + |\bx_i^\top(\hbx{\cM,i}-\hbi)|\\
        &\leq \polylog_1(n) + \polylog_5(n)\\
        &:= \polylog_6(n)\\
        |\ld_i(\hbx{\cM,i})|&\leq 1+C_y^s(n)+\polylog_6^s(n):=\polylog_7(n)\\
        \Vert\hbm - \hbx{\cM,i} \Vert&\leq \frac{1}{\lambda\nu}\Vert\bx_i\Vert\cdot|\ld_i(\hbx{\cM,i})|
        \leq \frac{2\sqrt{C_X}}{\lambda\nu}\polylog_7(n)\\
        |\bx_i^\top\hbm|&\leq |\bx_i^\top\hbx{\cM,i}| + |\bx_i^\top (\hbm - \hbx{\cM,i})|\\
        &\leq \polylog_6(n)+\frac{4C_X}{\lambda\nu}\polylog_7(n)\\
        &:=\polylog_8(n)\\
        |\ld_i(\hbm)|&\leq 1+C_y^s(n) + \polylog_8^s(n):=\polylog_9(n).
    \end{align*}

    \noindent\underline{Step 5}

    Now we are ready to bound events $F_2$ and $F_3$ as stated in this lemma. For the convenience of the readers we re-state the definitions of these two events:
    \begin{align*}
        F_2&=\{\max_{i\in[n]}|\ld_i(\hb)|>C_\ell(n)\},\\
        F_3&= \{\exists t\in[0,1], \cM\subset\cD \text{ with } |\cM|\leq m, \Vert\bldd(\bxi(t)) - \bldd(\hbm) \Vert>C_{\ell\ell}(n)\Vert\bxi(t)-\hbm \Vert\},
        \\ &\text{ where } \bxi(t)=t\hb+(1-t)\hbm,\; \bldd(\bbeta):=[\ldd_i(\bbeta)]_{i\in[n]}.
    \end{align*}
    Event $F_2$ was actually bounded in \eqref{eq: proof_lem_P(F_234)_F_2} with $C_\ell(n)=\polylog_3(n)$. Note that we allow $j=0$ in \eqref{eq: proof_lem_P(F_234)_F_2}, which indicates
    \[
        \PP(F_2)\leq \PP(\cup_{i=1}^3 E_i^c) \leq nq_n^{(y)} + 2n^{1-c} + ne^{-p/2} + e^{-p}.
    \]
    This proves part (a) of the lemma.
    For part (b), if $\bbeta\in\cB_{1,\cM}$, then it is a convex combination of $\hb$ and $\hbm$, so 
    \begin{align*}
        |\bx_i^\top\bbeta|
        \leq \max\{|\bx_i^\top\hb|, |\bx_i^\top\hbx{\cM}|\}\leq \max\{\polylog_3(n), \polylog_8(n)\}:=\polylog_9(n).
    \end{align*}
    Therefore 
    \begin{align*}
        |\lddd_i(\bbeta)|&\leq 1+C_y^s(n)+\polylog_9^s(n)
        :=\polylog_{10}(n)
        \label{eq: polylog10}\numberthis.
    \end{align*}
    In addition, notice that $\forall \bbeta\in\cB_{1,\cM}, \forall s\in[0,1], \bm{\eta}(s):=s\bbeta + (1-s)\hbm\in\cB_{1,\cM}$ as well. Therefore under $\cup_{i=1}^4$, $\forall |\cM|\leq m, \forall \bbeta\in\cB_{1,\cM},$
    \begin{align*}
        \Vert\bldd_{\backslash\cM}(\bbeta) - \bldd_{\backslash\cM}(\hbm) \Vert_2^2 
        &= \sum_{i\notin\cM}[\ldd_i(\bbeta) - \ldd_i(\hbm)]^2
        = \sum_{i\notin\cM} [\bar{\lddd}_i\bx_i^\top (\bbeta - \hbm)]^2,
    \end{align*}
    where $\bar{\lddd}_i:=\int_0^1 \lddd_i(s\bbeta+(1-s)\hbm) ds$,
    so
    \[
        |\bar{\lddd}_i| \leq \int_0^1 |\lddd_i(s\bbeta+(1-s)\hbm)| ds\leq \polylog_{10}(n),
    \]
    so
    \begin{align*}
        \Vert\bldd_{\backslash\cM}(\bbeta) - \bldd_{\backslash\cM}(\hbm) \Vert_2^2 
        &\leq \sum_{i\notin\cM} [\bar{\lddd}_i\bx_i^\top (\bbeta - \hbm)]^2\\
        &\leq \polylog_{10}^2(n) (\bbeta - \hbm)^\top \bXm^\top\bXm(\bbeta - \hbm)\\
        &\leq \polylog_{10}^2(n) (\sqrt{\gamma_0}+3)^2 C_X \Vert\bbeta - \hbm \Vert^2 
    \end{align*}
    Also, notice that $ \Vert\brdd(\bbeta) - \brdd(\hbm) \Vert\leq C_{rr}(n)\Vert\bbeta-\hbm \Vert$ is trivially satisfied by Assumption B2.

    Simialrly, for $F_4$, under $\cup_{i=1}^4 E_i^c$, $\forall |\cM|\leq m,  \forall i \notin \cM, \forall \bbeta\in \cB(\hbm, 1)$,
    \begin{align*}
        |\bx_i^\top\bbeta|&\leq |\bx_i^\top\hbm| + |\bx_i^\top(\bbeta-\hbm)|
            \leq \polylog_8(n) + 2\sqrt{C_X}\\
        |\lddd_i(\bbeta)|&\leq 1+C_y^s(n)+|\polylog_8(n) + 2\sqrt{C_X}|^s:=\polylog_{11}(n)
    \end{align*}
    It then follows that
    \begin{align*}
        \Vert\bldd_{\backslash\cM}(\bbeta_1) - \bldd_{\backslash\cM}(\bbeta_2) \Vert_2^2
        &\leq \sum_{i\notin\cM}\bar{\lddd}_i[\bx_i^\top(\bbeta_1-\bbeta_2)]^2
        \leq \polylog_{11}^2(n) (\sqrt{\gamma_0}+3)^2 C_X \Vert\bbeta_1 - \bbeta_2 \Vert^2.
    \end{align*}
    Finally, if we define $C_\ell(n):= \polylog_3(n) $, $C_{\ell\ell}(n)=\max\{\polylog_{10}(n), \polylog_{11}(n)\}$, we have
    \[
        \PP(F_1\cup F_2 \cup F_3\cup F_4)\leq \PP(\cup_{i=1}^4 E_i^c)\leq nq_n^{(y)} + 4n^{1-c} + ne^{-p/2} + e^{-p}.
    \]
    To make the $\polylog_k(n)$ terms explicit, we repeat them here:
    \begin{align*}
        \polylog_1(n) &= 2\sqrt{(\lambda\nu)^{-1}\gamma_0C_X(1+C_y^s(n))(1+c)\log(n)}\\
        \polylog_2(n) &= 1+C_y^s(n)+\polylog_1^s(n)\\
        \polylog_3(n)&=\polylog_1(n)+\frac{C_X}{\lambda\nu}\polylog_2(n)\\
        \polylog_4(n)&=1+C_y^s(n)+\polylog_3^s(n)\\
        \polylog_5(n)&= \frac{2(\sqrt{\gamma_0}+3)C_X}{\lambda\nu}\sqrt{m(2m+c)p^{-1}\log(n)}\polylog_4(n)\\
        \polylog_6(n)&=\polylog_1(n)+\polylog_5(n)
        \\
        \polylog_7(n)&=1+C_y^s(n)+\polylog_6^s(n)\\
        \polylog_8(n)&=\polylog_6(n)+\frac{4C_X}{\lambda\nu}\polylog_7(n)\\
        \polylog_9(n)&=1+C_y^s(n)+\polylog_8^s(n)\\
        \polylog_{10}(n)&=1+C_y^s(n)+\polylog_9^s(n)\\
        \polylog_{11}(n)&=1+C_y^s(n)+|\polylog_8(n) + 2\sqrt{C_X}|^s.
        \label{eq: def_all_polylogs}\numberthis
    \end{align*}
\end{proof}

\subsubsection{Proof of Lemma \ref{lem: P(F_5)}}
\label{ssec: proof_lem_P(F_5)}

\begin{proof}
    Recall that the goal is to bound
    \[
        \PP(F_5)=
        \PP\left( \max_{|\cM|\leq m}\Vert\bbXm\bGm^{-1}(\hbm)\bX_{\cM}^\top \Vert_{2,\infty} > 
         \frac{(\sqrt{\gamma_0}+3)C_X\vee 1}{\lambda\nu}
         \sqrt{\frac{m(1+4(c+1)\log(n))}{p}}  \right)
        .
    \]
    Notice that
    \begin{align*}
        &\max_{|\cM|\leq m}\Vert\bbXm\bGm^{-1}(\hbm)\bX_{\cM}^\top \Vert_{2,\infty}\\
        =&
        \max\left\{ \max_{|\cM|\leq m}\Vert\bGm^{-1}(\hbm)\bX_{\cM}^\top \Vert_{2,\infty}, \;\; \max_{|\cM|\leq m}\Vert\bXm\bGm^{-1}(\hbm)\bX_{\cM}^\top \Vert_{2,\infty} \right\},
    \end{align*}
     and we will bound the two terms separately.

         Let us fix a subset $\calM_0\subset [n]$ of size $|\calM_0|=s\le m$. Then $\bX_{\calM_0}$ is independent of $\bG_{\backslash\calM_0}^{-1}(\hat{\bbeta}_{\backslash\calM_0})$. That is,
         \[
         \bG_{\backslash\calM_0}^{-1}(\hat{\bbeta}_{\backslash\calM_0})
         \bX_{\calM_0}^{\top}
         =
         \bSigma_*^{1/2}\bZ_{\calM_0}^{\top}
         \]
         where $\bSigma_*=\bG_{\backslash\calM_0}^{-1}(\hat{\bbeta}_{\backslash\calM_0})\bSigma\bG_{\backslash\calM_0}^{-1}(\hat{\bbeta}_{\backslash\calM_0})$ satisfies
         \[
         \sigma_{\max}(\bSigma_*)
         \le \|\bG_{\backslash\calM_0}^{-1}(\hat{\bbeta}_{\backslash\calM_0})\|^2\times \|\bSigma\|
         \le \frac{C}{p(\lambda\nu)^2}
         \]
         almost surely, due to the $\nu$-strong convexity of the penalty function. Thus by Lemma~\ref{lem:gauss-mat-l4infty-norm}, since $|\calM_0|=s$, we have for any $c\ge 0$ that
         \[
         \PP\left(
         \|\bG_{\backslash\calM_0}^{-1}(\hat{\bbeta}_{\backslash\calM_0})
         \bX_{\calM_0}^{\top}\|_{2,\infty}
         \ge 
         \frac{1}{\lambda\nu}
         \sqrt{\frac{s(1+4(c+1)\log(n))}{p}}
         \right)
         \le \exp(-(c+1)s\log(n)).
         \]
         Now taking a union bound over all possible choices of $\calM_0$ such that $|\calM_0|\le s$, we have
         \begin{align*}
         &~\PP\left(
         \max_{|\calM|\le m}
         \|\bG_{\backslash\calM}^{-1}(\hat{\bbeta}_{\backslash\calM})
         \bX_{\calM}^{\top}\|_{2,\infty}
         \ge 
         \frac{1}{\lambda\nu}
         \sqrt{\frac{m(1+4(c+1)\log(n))}{p}}
         \right)\\
         \le &~
         \sum_{s=1}^m
         {n\choose s}
         \exp(-(c+1)s(\log(n)))\\
         \le &~
         \sum_{s=1}^m
         \exp(-(c+1)s\log(n))
         \le
         2
         n^{-c}.
         \end{align*}

         The proof technique for the second assertion is almost identical to the first part, with the following differences. We fix $\calM_0$ with $|\calM_0|=s\le m$, and just as in part i), we obtain:
         \begin{align*}
         &~\PP\left(
         \|\bX_{\backslash\calM_0}\bG_{\backslash\calM_0}^{-1}(\hat{\bbeta}_{\backslash\calM_0})
         \bX_{\calM_0}^{\top}\|_{2,\infty}
         \ge 
         \frac{C\|\bX_{\calM_0}\|}{\lambda\nu}
         \sqrt{\frac{s(1+4(c+1)\log(n))}{p}}
         \big\vert
         \bX_{\calM_0}
         \right)\\
         \le&~\exp(-(c+1)s\log(n)).
         \end{align*}
         Now by the independence of $\bX_{\calM_0}$ and $\bX_{\backslash\calM_0}$ for any fixed $\calM_0$, we can remove the conditioning on $\bX_{\backslash\calM_0}$ to write:
         \begin{align*}
         &~\PP\left(
         \|\bX_{\backslash\calM_0}\bG_{\backslash\calM_0}^{-1}(\hat{\bbeta}_{\backslash\calM_0})
         \bX_{\calM_0}^{\top}\|_{2,\infty}
         \ge 
         \frac{\|\bX_{\calM_0}\|}{\lambda\nu}
         \sqrt{\frac{s(1+4(c+1)\log(n))}{p}}
         \right)\\
         =&~\EE(\PP(\ldots|\bX_{\calM_0}))
         \le \exp(-(c+1)s\log(n)).
         \end{align*}
         Then taking the union bound over all possible $\calM_0$ as before, we have:
         \[
         \PP\left(
         \|\bX_{\backslash\calM_0}\bG_{\backslash\calM_0}^{-1}(\hat{\bbeta}_{\backslash\calM_0})
         \bX_{\calM_0}^{\top}\|_{2,\infty}
         \ge 
         \frac{\|\bX\|}{\lambda\nu}
         \sqrt{\frac{m(1+4(c+1)\log(n))}{p}}
         \right)
         \le 2n^{-c}.
         \]
         Now by Lemma~\ref{lem: norm_X}, the final bound becomes
         \[
         \PP\left(
         \|\bX_{\backslash\calM_0}\bG_{\backslash\calM_0}^{-1}(\hat{\bbeta}_{\backslash\calM_0})
         \bX_{\calM_0}^{\top}\|_{2,\infty}
         \ge 
         \frac{(\sqrt{\gamma_0}+3)C_X}{\lambda\nu}
         \sqrt{\frac{m(1+4(c+1)\log(n))}{p}}
         \right)
         \le 2n^{-c}+{\rm e}^{-p}.
         \]

     Finally, by taking the maximum of the bounds obtained in the two parts above and a union bound, we have
     \begin{align*}
         &\PP\left( \max_{|\cM|\leq m}\Vert\bbXm\bGm^{-1}(\hbm)\bX_{\cM}^\top \Vert_{2,\infty} > 
         \frac{(\sqrt{\gamma_0}+3)C_X\vee 1}{\lambda\nu}
         \sqrt{\frac{m(1+4(c+1)\log(n))}{p}}  \right)\\
         \leq& ~4n^{-c} + e^{-p}.
     \end{align*}
\end{proof}

\subsubsection{$\ell_2$ error of multiple Newton steps}
\label{sssec: multi-step}
To study multiple Newton steps, we first prove quadratic convergence for the Newton method in a general setting:
\begin{lemma}\label{lem:t step newton general}
    Suppose $\bbf(\bbeta):\RR^p\to\RR^p$ has Jacobian $\bG(\bbeta)$ with $\lambda_{\min}(\bG(\bbeta))\geq \nu$ for all $\bbeta$, and suppose $\bbf(\bbeta)=\bzero$ has a unique solution $\bbeta^*$. Suppose $\{\bbeta^{(t)}\}_{t\geq 1}$ is the path of Newton method in searching for $\bbeta^*$, i.e. $\forall t\geq 2$,
    \[
        \bbeta^{(t)} := \bbeta^{(t-1)} - \bG^{-1}(\bbeta^{(t-1)})\bbf(\bbeta^{(t-1)}).
    \]
    Let $r_{t,n}:=\Vert\bbeta^{(t)}-\bbeta^* \Vert$ be the $l_2$ error of the $t^{\rm th}$ step. If $\forall \bx_1,\bx_2\in B(\bx^*,r_1)$, $\Vert\bG(\bx_1)-\bG(\bx_2) \Vert\leq C \Vert\bx_1-\bx_2 \Vert$, then 
    \[
        r_{t,n}\leq \frac{C}{2\nu}r_{t-1,n}^2.
    \]
    Consequently, 
    \[
        r_{t,n}\leq \left(\frac{C}{2\nu}\right)^{2^{t-2}}   r_1^{2^{t-1}}. 
    \]
\end{lemma}

\begin{proof}
    By the definition of Newton steps,
    \begin{align*}
        \bbeta^{(t)} - \bbeta^*
        &= \bbeta^{(t-1)} - \bbeta^* - \bG^{-1}(\bbeta^{(t-1)})\bbf(\bbeta^{(t-1)}).
    \end{align*}
    Notice that we can add $\bG^{-1}(\bbeta^{(t-1)})\bbf(\bbeta^*)$ to the right hand side because $\bbf(\bbeta^*)=0$, so we have
    \begin{align*}
        \bbeta^{(t)} - \bbeta^*
        &= \bbeta^{(t-1)} - \bbeta^* - \bG^{-1}(\bbeta^{(t-1)})\bbf(\bbeta^{(t-1)}) + \bG^{-1}(\bbeta^{(t-1)})\bbf(\bbeta^*)\\
        &= \bbeta^{(t-1)} - \bbeta^* - \bG^{-1}(\bbeta^{(t-1)})[\bbf(\bbeta^{(t-1)})-\bbf(\bbeta^*)]\\
        &= \bbeta^{(t-1)} - \bbeta^* - \bG^{-1}(\bbeta^{(t-1)})\bar{\bG}(\bbeta^{(t-1)}-\bbeta^*)\\
        &=\bG^{-1}(\bbeta^{(t-1)}) [\bG(\bbeta^{(t-1)}) - \bar{\bG}](\bbeta^{(t-1)}-\bbeta^*),
    \end{align*}
    where in the penultimate step we used Taylor expansion with
    \[
        \bar{\bG}:=\int_0^1 \bG(a\bbeta^{(t-1)} + (1-a)\bbeta^* )  da,
    \]
    and the last step uses the trick that $\bbeta^{(t-1)} = \bG^{-1}(\bbeta^{(t-1)})\bG(\bbeta^{(t-1)})\bbeta^{(t-1)}$. Notice that
    \begin{align*}
        \Vert \bG(\bbeta^{(t-1)}) - \bar{\bG}\Vert 
        &= \Vert\int_0^1[\bG(\bbeta^{(t-1)}) - \bG(a\bbeta^{(t-1)} + (1-a)\bbeta^* )]da\Vert\\
        &\leq \int_0^1 \Vert\bG(\bbeta^{(t-1)}) - \bG(a\bbeta^{(t-1)} + (1-a)\bbeta^* ) \Vert da\\
        &\leq \int_0^1 C \Vert(1-a)(\bbeta^{(t-1)}-\bbeta^*) \Vert da\\
        &= C \Vert\bbeta^{(t-1)}-\bbeta^* \Vert\int_0^1 (1-a)da\\
        &= \frac{C}{2}\Vert\bbeta^{(t-1)}-\bbeta^* \Vert
    \end{align*}
    Therefore we have
    \begin{align*}
        r_{t,n}=\Vert\bbeta^{(t)} - \bbeta^*\Vert
        &\leq \Vert\bG^{-1}(\bbeta^{(t-1)})  \Vert\cdot \Vert \bG(\bbeta^{(t-1)}) - \bar{\bG}\Vert\cdot \Vert\bbeta^{(t-1)}-\bbeta^* \Vert\\
        &\leq \frac{1}{\nu} \frac{C}{2}\Vert\bbeta^{(t-1)}-\bbeta^* \Vert^2
        = \frac{C}{2\nu}r_{t-1,n}^2.
    \end{align*}
    We then immediately have
    \[
        r_{t,n} = \left(\frac{C}{2\nu}\right)^{2^{t-2}} r_{1,n}^{2^{t-1}}.
    \]
    where $r_{1,n} = C_1(n)\sqrt{\frac{m^3}{n}}$ is the bound we obtained in Lemma \ref{lem:l2-norm-diff}.
\end{proof}

It is then a direct application of Lemma \ref{lem:t step newton general} to prove Theorem \ref{thm: l2_error_multistep}:

\begin{proof}[Proof of Theorem \ref{thm: l2_error_multistep}]
    It is a direct application of Lemma \ref{lem:t step newton general} together with several previous results.

    First, under event $F^c$, by Lemma \ref{lem:beta_lo_error}, $r_1:=\Vert\tbm^{(1)}-\hbm\Vert\leq C_1(n)m^{\frac32}n^{-\frac12}=o(1)$ since $m=o(n^{1/3})$. Then for the Lipschiz condition of $\bGm(\bbeta)$, notice that under event $F_4^c$, for large enough $n$, $\forall \bbeta_1,\bbeta_2\in\cB(\hbm,r_1)\subset\cB(\hbm,1)$,
    \begin{align*}
        \bGm(\bbeta_1) - \bGm(\bbeta_2) 
        &= \bXm^\top[\Ldd_{\backslash\cM}(\bbeta_1)-\Ldd_{\backslash\cM}(\bbeta_2)]\bXm+\lambda r[\Rdd(\bbeta_1)-\Rdd(\bbeta_2)]\\
        &=\bbXm^\top\bGamma\bbXm
    \end{align*}
    where the last line is similar to \eqref{eq:def_bGamma}.
    Under $F_4^c$, we have $\Vert\bGamma \Vert_2\leq\Vert\bGamma \Vert_{Fr}\leq (C_{\ell\ell}(n)+\lambda C_{rr}(n))\Vert\bbeta_1-\bbeta_2 \Vert$, so 
    \begin{align*}
        \Vert\bGm(\bbeta_1) - \bGm(\bbeta_2)\Vert
        &\leq \Vert\bbXm \Vert^2 (C_{\ell\ell}(n)+\lambda C_{rr}(n))\Vert\bbeta_1-\bbeta_2 \Vert\\
        &\leq (1+(\sqrt{\gamma_0}+3)\sqrt{C_X})^2(C_{\ell\ell}(n)+\lambda C_{rr}(n))\Vert\bbeta_1-\bbeta_2 \Vert\\
        &:=C_2(n)\Vert\bbeta_1-\bbeta_2 \Vert.
    \end{align*}
    We then apply Lemma \ref{lem:t step newton general} with the Lipschitz constant $C=C_2(n)=O(\polylog(n))$ and the strong convexity constant to be $\lambda\nu$.
\end{proof}

\subsection{Proof of Theorem \ref{thm: main_epscert} and Theorem \ref{thm: main_accuracy}}
\label{ssec: proof_main}
\begin{proof}[Proof of Theorem \ref{thm: main_epscert}]
    Recall that we defined 
    \[
        \cX_{r_{t,n}}^{(t)}:=\{\cD: \underset{|\cM|\leq m}{\max}\Vert\tbm^{(t)}-\hbm\Vert_2\leq r_{t,n}\}.
    \]
    
    By Lemma \ref{lem: direct_perturbation}, $\tbm^{R,t} = \tbm^{(t)} +\bb$ is $(\phi_n,\epsilon)$-PAR if $\PP(\cX_{r_{t,n}}^{(t)})\geq 1-\phi_n$.
    
    By Lemma \ref{lem:l2-norm-diff}, $\left(\cX_{r_{t,n}}^{(t)}\right)^c\subset F=\cup_{i=1}^5 F_i$ defined in \eqref{eq: def_failure_events}.

    By Lemma \ref{lem: P(F)},
    \[
        \PP(F)\leq nq_n^{(y)} + 8n^{1-c} + ne^{-p/2} + 2e^{-p} := \phi_n,
    \]
    as long as we set the constants $C_1$, $C_\ell(n)$, $C_{\ell\ell}(n)$, $C_{rr}(n)$ in the definion of $F$ \eqref{eq: def_failure_events} as in Lemma \ref{lem: P(F)}, so that
    \begin{align*}
        r_{t,n} &= [C_1(n)]^{2^{t-1}}\left(\frac{C_2(n)m^3}{2\lambda\nu n}\right)^{2^{t-2}},\\
        C_1(n) &= \frac{2\sqrt{3}}{3\lambda^2\nu^2}[C_{\ell\ell}(n)+\lambda C_{rr}(n)]C_1(C_1+1)C_\ell^2(n)C_{xx}(n),\\
        C_2(n) &= (1+(\sqrt{\gamma_0}+3)\sqrt{C_X})^2(C_{\ell\ell}(n)+\lambda C_{rr}(n)).
    \end{align*}

    Therefore we conclude that $\tbm^{R,t}= \tbm^{(t)} + \bb$ achieves $(\phi_n,\epsilon)$-PAR with 
    \[
        \phi_n = nq_n^{(y)} + 8n^{1-c} + ne^{-p/2} + 2e^{-p},
    \]
    if $\bb$ has density $p(\bb)\propto e^{-\frac{\eps}{r_{t,n}}\Vert\bb\Vert}$.
\end{proof}

\medskip

\begin{proof}[Proof of Theorem \ref{thm: main_accuracy}]
    First notice that by Taylor expansion,
    \begin{align*}
        |\ell_0(\tbm^{(t)}+\bb)-\ell_0(\hbm)|
        &\leq |\bar{\ld}_0\Vert\bx_0^\top(\tbm^{(t)}-\hbm+\bb)|.
    \end{align*}
    Let $F$ be the failure event defined in \eqref{eq: def_failure_events}. 
    Recall that by Equation \eqref{eq: bd_all_betahat} and \eqref{eq: proof_lem_P(F_234)_1}, under $F^c$, 
    \begin{align*}
        \Vert\hb \Vert&\leq \sqrt{(\lambda\nu)^{-1}(1+C_y^s(n))n}\\
        \max_{|\calM|\le m}
        \Vert\hb-\hbm\Vert&\leq \frac{\sqrt{m}}{\lambda\nu}(\sqrt{\gamma_0}+3)\sqrt{C_X}\polylog_4(n).
    \end{align*}
    Define events
    \begin{align*}
        E_5&:=\left\{ |\bx_0^\top\hb| \leq \sqrt{(\lambda\nu)^{-1}C_X\gamma_0(1+C_y^s(n)2c\log(n))} \right\}\\
        E_6&:=\left\{  
        \max_{|\cM|\leq m}
        |\bx_0^\top(\hb-\hbm)|\leq\frac{2C_X}{\lambda\nu}(\sqrt{\gamma_0}+3)\sqrt{\frac{m(2m+c)}{p}\log(n)}\polylog_4(n)  
        \right\}.
    \end{align*}
    Then for any $\cD\in F^c$, we have 
    \begin{align*}
        \PP(E_5^c|\cD)&\leq \PP\left(|\bx_0^\top\hb|\leq \sqrt{\frac{C_X}{p}}\Vert\hb \Vert\sqrt{2c\log(n)}|\cD\right)\leq 2n^{-c}\\
        \PP(E_6^c |\cD)&\leq \PP\left(\left. 
 \exists|\cM|\leq m, |\bx_0^\top(\hb-\hbm)|>\sqrt{\frac{C_X}{p}}\Vert\hb-\hbm \Vert\cdot 2\sqrt{\log(N)+c\log(n)} \right|\cD  \right)\\
    &\leq 2n^{-c},
    \end{align*}
    where $N=\sum_{s=0}^m{n\choose s}\leq 2{n\choose m}$, so 
    \[
        \log(N)\leq \log(2)+m\log(en/m)\leq 2m\log(n).
    \]
    Under $E_5\cap E_6$, 
    \[
        |\bx_0^\top\hbm|\leq |\bx_0^\top\hb|+|\bx_0^\top(\hb-\hbm)|\leq \polylog_{12}(n)
    \]
    if $m=o(\sqrt{n})$.
    
    Define 
    \begin{align*}
        E_7&:= \left\{ \forall|\cM|\leq m, |\bx_0^\top(\tbm^{(t)}-\hbm+\bb)|\leq 2\sqrt{C_X}
        \left(\frac{2\sqrt{p}}{\epsilon}+\frac{1}{\sqrt{p}}\right)r_{t,n}\cdot \sqrt{(2m+c)\log(n)} \right\},\\
        E_8&:=\{ \Vert\bb\Vert\leq\frac{2p}{\epsilon}r_{t,n} ,|y_0|\leq C_y(n)\}.
    \end{align*}
    Under $F^c$,
    \(
        \Vert\tbm^{(t)}-\hbm \Vert\leq 
        r_{t,n}
    \).  Then for any $\cD\in F^c$
    \begin{align*}
        \PP(E_7^c\cap E_8|\cD)
        &\leq \PP(\exists|\cM|\leq m: |\bx_0^\top(\tbm^{(t)}-\hbm+\bb)| >\\
        &\qquad\qquad\sqrt{\frac{C_X}{p}}\Vert\tbm^{(t)}-\hbm+\bb \Vert\cdot 2\sqrt{\log(N)+c\log(n)}|\cD)\\
        &\leq 2n^{-c}.
    \end{align*}
    Under $(\cap_{i=5}^8 E_i)$, 
    $\forall a\in[0,1]$,
    \begin{align*}
        &~|\bx_0^\top[a(\tbm^{(t)}+\bb)+(1-a)\hbm]|\\
        \leq&~ |\bx_0^\top\hbm| +a|\bx_0^\top(\tbm^{(t)}-\hbm+\bb)|\\
        \leq&~ \polylog_{12}(n)+2\sqrt{C_X}\left(\frac{2\sqrt{p}}{\epsilon}+\frac{1}{\sqrt{p}}\right)r_{t,n}\cdot \sqrt{(2m+c)\log(n)}\\
        \leq&~ \polylog_{13}(n),
    \end{align*}
    provided $r_{t,n}=o\left(\frac{\epsilon}{\sqrt{mp\polylog(n)}}\right)$ so that the second term is $O(\polylog(n))$. Thus, for any $\cD\in F^c$,
    \begin{align*}\label{eq:ged-eq1}
        &~\EE\left[\int|\bx_0^\top[a(\tbm^{(t)}+\bb)+(1-a)\hbm]|^{2s}da|\cD\right]\\
        \le&~
        (\polylog_{13}(n))^{2s}\\
        &~+
        C_s(\EE\|\bx_0\|^{4s})^{1/2}(\|\hbm\|^{4s}+\|\tbm^{(t)}-\hbm\|^{4s}+\EE|\bb|^{4s})^{1/2}
        \PP((\cap_{i=5}^8 E_i)^c|\cD)\\
        \le&~
        (\polylog_{13}(n))^{2s}+
        C_sC_X
        \left(\frac{2p}{\epsilon}+1\right)^{2s}r_{t,n}^{2s}\cdot (\sqrt{(2m+2s)\log(n)})^{2s}
        \cdot n^{-2s}\\
        \le&~
        \polylog_{14}(n)\numberthis
    \end{align*}
    using events $\cap_{i=5}^8 E_i$ for some $c\ge 2s$. Similarly we obtain
    \begin{equation}\label{eq:ged-eq2}
        \EE\left[|\bx_0^\top[(\tbm^{(t)}-\hbm)+\bb]|^{2}da|\cD\in F^c\right]
        \le 
        2
        C_X\left(\frac{2\sqrt{p}}{\epsilon}+\frac{1}{\sqrt{p}}\right)^2r_{t,n}^2\cdot (2m+2)(\log(n))
    \end{equation}
    using events $\cap_{i=5}^8 E_i$ for some $c\ge 2$. Finally we have that, 
    \begin{align*}
        {\rm GED}^\epsilon(\tbm^{R,t},\hbm)
        =
        &~\EE\left(|\ell_0(\tbm^{(t)}+\bb)-\ell_0(\hbm)|\vert\cD\right)\\
        \leq& ~
        \left(\EE(|\bar{\ld}_0|^2|\cD)\right)^{1/2}
        \left(\EE(\vert\bx_0^\top\bb +\bx_0^{\top}(\tbm^{(t)}-\hbm)\vert^2|\cD)|\right)^{1/2}
    \end{align*}
    Now using the previous inequalities, we will bound each of the quantities on the right hand side of the latest display. First notice that by \eqref{eq:ged-eq1}, for any $\cD\in F^c$
    \begin{align*}
        &~\EE(|\bar{\ld}_0|^2|\cD)\\
        =&~
        \EE\left(\vert
        \int_0^1 \ld_0((\tbm^{(t)}+\bb)+(1-a)\hbm) da 
        \vert^2|\cD\right)\\
         \le&~
        3C^2\left(1+\EE|y_0|^{2s}+\EE\left[\int|\bx_0^\top[a(\tbm^{(t)}+\bb)+(1-a)\hbm]|^{2s}da|\cD\right]\right)\\
        \le&~
        3C^2(1+C_{y,s}+\polylog_{14}(n))
    \end{align*}
    where we use Assumption~\ref{assum:y} for the last two inequalities.  Similarly using \eqref{eq:ged-eq2} we have for any $\cD\in F^c$ that
    \begin{align*}
        &~{\rm GED}^\epsilon(\tbm^{R,t},\hbm)\\
        =
        &~\EE\left(|\ell_0(\tbm^{(t)}+\bb)-\ell_0(\hbm)|\vert\cD\right)
         \\
        \le&~
        3C(1+C_{y,s}+\polylog_{14}(n))^{1/2}
        \sqrt{C_X}\left(\frac{2\sqrt{p}}{\epsilon}+\frac{1}{\sqrt{p}}\right)r_{t,n}\cdot \sqrt{(2m+2)(\log(n))}
        .
    \end{align*}    
    This proves the first part of the theorem. 
    
    To prove the second part, notice that for any $\cD\in F^c$, ${\rm GED}^\epsilon(\tbm^{R,t},\hbm)=O(1)$ if $r_{t,n}=o\left(\frac{\epsilon}{\sqrt{mp\polylog(n)}}\right)$. To find the smallest $t$ such that this holds true, define $\alpha = \log(m+1)/\log(n)$, then we have
    \begin{align*}
        r_{t,n} &= [C_1(n)]^{2^{t-1}}\left(\frac{C_2(n)(m+1)^3}{2\lambda\nu n}\right)^{2^{t-2}}
        =\polylog(n) n^{(3\alpha-1)2^{t-2}}.
    \end{align*}
    We need 
    \begin{align*}
        r_{t,n} &= o\left(\frac{\epsilon}{\sqrt{mp\polylog(n)}}\right)
         = o(n^{-\frac12 (\alpha+1)}),
    \end{align*}
    which is satisfied when 
    \[
        t > T = 1+\log_2\left(\frac{\alpha+1}{1-3\alpha}\right).
    \]
\end{proof}

\subsection{An example for logistic ridge}
Here I provide explicit formulae for the constants, for logistic ridge:
\[
    C_X=1, s=1, C_y(n)=1, \nu=1, C_{rr}(n)=0.
\]
and WLOG assume $c=3$ ($c$ is the only custom constant throughout the paper to control the convergence speed of $\phi_n\to0$).
We then have
\begin{align*}
    \polylog_1(n) &= 4\sqrt{2\gamma_0\lambda^{-1}\log(n)}\\
    C_\ell(n) &= \polylog_1(n) + 8\lambda^{-1}\\
    \polylog_5(n) &= 4(\sqrt{\gamma_0}+3)\lambda^{-1}\sqrt{m(2m+3)p^{-1}\log(n)}\\
    C_{\ell\ell}(n) &= 4+8\lambda^{-1}+\polylog_1(n)+\polylog_5(n)\\
    C_{xx}(n) &= (\sqrt{\gamma_0}+3)\lambda^{-1}\sqrt{mp^{-1}(1+16\log(n))}\\
    \\
    C_2(n) &= [\sqrt{\gamma_0}+4]^2C_{\ell\ell}(n)\\
    C_1(n)&=\frac{2}{\sqrt{3}\lambda^2}C_2(n)C_{\ell}^2(n)C_{xx}(n)
\end{align*}
Now the constants only depend on $\lambda, \gamma_0$ except $\polylog_5(n)$ and $C_{xx}(n)$, which technically also depend on $m,p$. But they can nonetheless be calculated.

\section{Technical Lemmas}
\begin{lemma}
\label{lem: conc_single_x}
    \(
        \PP(\|\bx\|>2c\rho_{\max}\sqrt{\log(n)})\leq \frac{1}{\sqrt{2\pi c}}n^{-c}
    \)
    for $n\geq 3$.
\end{lemma}
\begin{lemma}\label{lem: max_of_gaussian}
    Let  $Z_i$ be $N$ dependent $\cN(0,\sigma_i^2)$ random variables. Suppose $\sigma_i^2\leq \sigma_{\max}^2<\infty$. Then $\forall n\geq 1, c>0$:
    \[
        \PP( \max_{i\in[N]}|Z_i|> 2\sigma_{\max}\sqrt{\log(N)+c\log(n)} )\leq 2n^{-c}.
    \]
\end{lemma}

\begin{proof}
    Let $Z:=\max_{i\in[N]}Z_i$.
    By Lemma \ref{lem: boreltis},
    \[
        \PP(Z-\EE Z>t)\leq e^{-\frac{t^2}{2\sigma_{\max}^2}}.
    \]
    By Lemma \ref{lem: chatterjee14}, $\EE Z\leq \sigma_{\max}\sqrt{2\log(N)}$. Setting $t=\sigma_{\max}\sqrt{2c\log(n)}$ yields
    \[
        \PP(Z > \sigma_{\max}(\sqrt{2\log(N)}+\sqrt{2c\log(n)}))\leq n^{-c}.
    \]
    Finally notice that 
    \[
        \max_{i\in[N]} |Z_i| = \max\left\{\max_{i\in[N]} Z_i, \max_{i\in[N]} (-Z_i)\right\},
    \]
    so by a union bound and the fact that $\sqrt{a}+\sqrt{b}\leq \sqrt{2(a+b)}$ we have
    \[
        \PP( \max_{i\in[N]}|Z_i|> 2\sigma_{\max}\sqrt{\log(N)+c\log(n)} )\leq 2n^{-c}.
    \]
\end{proof}

\begin{lemma}
\label{lem: laplace and gamma}
Suppose $\bb\in\RR^p$ is a random vector with 
density 
\[
    p_{\bb}(\bb) \propto e^{-C\Vert\bb\Vert },
\]
then $\Vert \bb\Vert \sim \Gamma(p,C)$ with density
\[
    p_{\Vert \bb\Vert}(r) = \frac{C^p}{\Gamma(p)}r^{p-1}e^{-cr},
\text{ and }
    \PP\left(\Vert \bb\Vert > \frac{2p}{C}\right)\leq e^{-(1-\log(2))p}.
\]
\end{lemma}

\medskip

\begin{proof}

  Note that the pdf of $\bb$ depends only on its $\ell_2$ norm $\|\bb\|$. Thus, $\bb$ has a spherically symmetric distribution. By Theorem 4.2 of \cite{fourdrinier2018shrinkage} the pdf of $\|\bb\|$ is given by:
  \[
  p_{\|\bb\|}(r)
  \propto\frac{2\pi^{p/2}}{\Gamma(p/2)}r^{p-1}{\rm e}^{-Cr}
  \]
  and thus $\|\bb\|\sim \Gamma(p,C)$. This finishes the proof of the first part. For the second part we use the moment generating function of the Gamma distribution to write:
  \begin{align*}
      \PP(\|\bb\|>\frac{2p}{C})
      \le&~ \inf_{0<t<C}\EE({\rm e}^{t\|\bb\|}){\rm e}^{-\frac{2tp}{C}}
      =\inf_{0<t<C} \left(1-\frac{t}{C}\right)^{-p}{\rm e}^{-\frac{2tp}{C}}
      =~2^p{\rm e}^{-p}.
  \end{align*}
  The last equality follows since the infimum in the previous line is achieved by $t=\frac{C}{2}$.
  \end{proof}

\begin{lemma}\label{lem:beta_lo_error}
    Under Assumptions~\ref{assum:separability}, \ref{assum:smoothness}, and \ref{assum:convexity}, 
    \begin{align*}
        \hbm-\hb &= \bbG^{-1} \left(\sum_{i\in\cM} \ld_i(\hbm)\bx_i\right)
        = \bbGm^{-1}\left(\sum_{i\in\cM} \ld_i(\hb)\bx_i\right)
    \end{align*}

    where 
    \begin{align*}
        \bbG&:= \int_0^1 \bG(t\hb + (1-t)\hbm)dt\\
        \bbGm&:=  \int_0^1 \bGm(t\hb + (1-t)\hbm)dt\\
        \bG(\bm{\beta}) &:= \bX^\top \diag[\ldd_i(\bm{\beta})]_{i\in[n]}\bX + \lambda \diag[\rdd_k(\bm{\beta})]_{k\in[p]} \\
        \bGm(\bm{\beta}) &:= \bXm^\top \diag[\ldd_i(\bm{\beta})]_{i\notin\cM}\bXm + \lambda \diag[\rdd_k(\bm{\beta})]_{k\in[p]}
    \end{align*}
\end{lemma}
\begin{proof}
    Consider the optimality conditions of $\hb$ and $\hbm$:
    \begin{align*}
        \sum_{i\notin\cM}\ld_i(\hbm)\bx_i + \lambda 
        \nabla r(\hbm) &= 0\\
        \sum_{i\in[n]}\ld_i(\hb)\bx_i + \lambda \nabla r(\hb) &= 0
    \end{align*}
    Subtracting one from another and applying the mean value theorem, we get
    \begin{align*}
        \bbG\cdot(\hbm - \hb)= \sum_{i\in\cM}\ld_i(\hbm)\bx_i.
        \label{eq: lem:beta_lo_error_1}\numberthis
    \end{align*}
    Multiplying $\bbG^{-1}$:
    \[
        \hbm - \hb = \bbG^{-1} \left(\sum_{i\in\cM} \ld_i(\hbm)\bx_i\right).
    \]
    For the other version, rearranging the terms in we get \eqref{eq: lem:beta_lo_error_1}:
    \[
        \bbGm(\hbm - \hb)= \sum_{i\in\cM}\ld_i(\hb)\bx_i,
    \]
    therefore
    \[
        \hbm - \hb = \bbGm^{-1}\left(\sum_{i\in\cM} \ld_i(\hb)\bx_i\right).
    \]
\end{proof}

\begin{lemma}\label{lem:gauss-mat-l4infty-norm}
    Let $\bX^{\top}=\bSigma^{1/2}\bZ$ be a $p\times m$ matrix, where $\bZ\in \RR^{p \times m}$ has elements $Z_{ij}\stackrel{iid}{\sim}\mathcal{N}(0,1)$ and $\bSigma$ is positive semi-definite with $\sigma_{\max}(\bSigma):=\rho_{\max}$. Then, for any $c\ge 1$ 

    \[
    \PP\left(
    \|\bX^{\top}\|_{2,\infty}
    \ge \sqrt{m\rho_{\max}(1+4c\log(n))}
    \right)
    \le \exp\left(-cm\log(n)\right).
    \]

\end{lemma}

\medskip

\begin{proof}[Proof of Lemma~\ref{lem:gauss-mat-l4infty-norm}] 


    Note that by a well-known equality for the $2\to\infty$ matrix norm, we have
    \[
    \|\bX^{\top}\|_{2,\infty}
    =\max_{1\le k\le p} \|\ev_k^{\top}\bX^{\top}\|_2
    =\max_{1\le k\le p} \|\ev_k^{\top}\bSigma^{1/2}\bZ\|_2.
    \]
    Note that $\ev_k^{\top}\bSigma^{1/2}\bZ\sim \mathcal{N}(0,\sigma_{kk}\II_m)$ for $1\le k\le p$, where $\sigma_{kk}$ is the $(k,k)$ element of $\bSigma$ and thus,
    \[
    \PP(\|\bX^{\top}\|_{2,\infty}\ge \sqrt{m\rho_{\max}(1+4c\log(n))})
    \le \exp(-cm\log(n))
    \]
    by concentration inequalities for $\chi^2$ distributed random variables \citep[see, e.g.,][]{laurent2000adaptive}. 
 \end{proof}

\subsection{Adopted Lemmata}
\begin{lemma}[Lemma 14 of \cite{auddy24a}]
\label{lem: stirling}
    For $1< s\leq N $ and $N> 2$, we have 
    \begin{enumerate}
        \item \[
    {N\choose s}\le 
    e^{s\log\frac{eN}{s}}
    \]
        \item For $m\leq (n+1)/3$,
        \[
            \sum_{s=0}^m {n\choose s}\leq 2{n\choose m}
        \]
    \end{enumerate}
\end{lemma}
\begin{proof} We prove each part separately:
    \begin{itemize}
        \item The first part is exactly Lemma 14 of \cite{auddy24a}.
        \item Notice that for $s<m<n$,
         \end{itemize}
        \begin{align*}
            \frac{{n\choose s}}{{n\choose m}}
            = \frac{\frac{n!}{s!(n-s)!}}{\frac{n!}{m!(n-m)!}}
            &=\frac{m(m-1)\cdots(s+1)}{(n-s)(n-s-1)\cdots(n-m+1)}
            \leq \left(\frac{m}{n-m+1}\right)^{m-s},
        \end{align*}
        so
        \begin{align*}
            \frac{\sum_{s=0}^m {n\choose s}}{{n\choose m}}
            &\leq 1+\sum_{s=0}^{m-1}\left(\frac{m}{n-m+1}\right)^{m-s}
            \leq \sum_{t=0}^{\infty}\left(\frac{m}{n-m+1}\right)^t
            = \frac{n-m+1}{n-2m+1}\\
            &\leq 1+\frac{m}{n-2m+1} \leq 2
        \end{align*}
        provided $\frac{m}{n-2m+1}\leq 1 \Leftrightarrow m\leq \frac{n+1}{3}$.
\end{proof}

\begin{lemma}[Lemma 19 in \cite{auddy24a}]
\label{lem: norm_X}
Suppose the rows of $\bX$ satisfy Assumption \ref{assum:normality}, then 
\[
    \PP (\|\bX^{\top} \bX\| \geq (\sqrt{\gamma_0} + 3 )^2 C_X) \leq {\rm e}^{-p}. 
\]

\end{lemma}
\begin{lemma}[Lemma 17 in \cite{auddy24a}]
\label{lem: max_norm_x_i}
    Let $\bx_1,\dots,\bx_n\stackrel{iid}{\sim} N(0,\bSigma)\in \RR^{p }$ and suppose $\rho_{\max}(\bSigma)\leq p^{-1}C_X$ for some constant $C_X>0$, then 
    \[
        \PP (\max_{1\le i \le n}\|\bx_i\|\geq 2\sqrt{C_X} ) \leq ne^{-p/2}.
    \]
\end{lemma}

\begin{lemma}[Lemma 4.10 of \cite{chatterjee2014}]
\label{lem: chatterjee14}
    Let $Z_i$ be $N$ dependent $\cN(0,\sigma_i^2)$ random variables with $\sigma_i\leq \sigma_{\max}<\infty$, then 
    \[
        \EE \max_{i\in[N]}Z_i \leq \sigma_{\max}\sqrt{2\log(N)}.
    \]
\end{lemma}
\begin{proof}
    The proof is essentially the same as Lemma 4.10 of \cite{chatterjee2014}, except for the absolute value.
    Let $Z:= \max_{i\in[N]}Z_i$. 
    By Jensen's inequality, $\forall t>0$:
    \begin{align*}
        e^{t\EE Z}\leq \EE e^{tZ}
        &= \EE e^{t\max_{i\in[N]} Z_i}
        \leq \sum_{i\in[N]}\EE e^{tZ_i}
        = \sum_{i\in[N]}e^{\frac12t^2\sigma_i^2}
        \leq N e^{\frac12 t^2 \sigma_{\max}^2},
    \end{align*}
    which implies
    \[
        \EE Z \leq \frac1t[\log(N)+\frac12 t^2 \sigma_{\max}^2].
    \]
    The right hand side minimizes at $t^2 = \frac{2\log(N)}{\sigma_{\max}^2}$, and hence we have
    \(
        \EE Z \leq \sigma_{\max}\sqrt{2\log(N)}.
    \)
\end{proof}
\begin{lemma}[Borel-TIS inequality, Theorem 2.1.1 of \cite{adler2009}]
\label{lem: boreltis}
    Let $Z_i$ be $N$ dependent $\cN(0,\sigma_i^2)$ random variables with $\sigma_i\leq \sigma_{\max}<\infty$, and $Z:= \max_{i\in[N]}Z_i$, then $\forall t>0$,
    \[
        \PP(Z-\EE Z>t)\leq e^{-\frac{t^2}{2\sigma_{\max}^2}}.
    \]
\end{lemma}

\begin{lemma} [Lemma 6 of \cite{jalali2016}] \label{lem:chi:sq:ind}
    Let $\bx \sim \cN(0,\bI_p)$, then
    \[
    \PP (\bx^{\top } \bx\geq p + pt ) \leq {\rm e}^{-    \frac{p}{2} (t- \log (1+t)) }.
    \]
\end{lemma}




\end{document}